\documentclass{article}

\usepackage{microtype}
\usepackage{graphicx}
\usepackage{subfigure}
\usepackage{booktabs} 

\usepackage[]{hyperref}

\usepackage[accepted]{icml2021}

\usepackage[inline]{enumitem}

\usepackage{amsmath}
\usepackage{amssymb}
\usepackage{amsfonts}
\usepackage{amsthm}
\usepackage{pifont}  
\usepackage{booktabs}       

\usepackage[makeroom]{cancel}

\usepackage{scrextend}

\usepackage{xspace}
\DeclareRobustCommand{\eg}{e.g.,\@\xspace}
\DeclareRobustCommand{\ie}{i.e.,\@\xspace}

\DeclareRobustCommand{\wrt}{w.r.t.\@\xspace}

\def \regret {\mathcal{R}}
\def \horizon {H}
\newcommand{\Sdist}[1]{ \rho_{\cX}\left(#1\right) }
\newcommand{\Adist}[1]{ \rho_{\cA}\left(#1\right) }
\newcommand{\Wassdist}[1]{ W_1\left(#1\right) }

\def \Sdistfunc {\rho_{\cX} }
\def \Adistfunc {\rho_{\cA} }
\def \distfunc {\rho }

\newcommand{\dist}[1]{ \rho\left[#1\right] }
\newcommand{\dirac}[1]{ \delta_{#1} }

\def \bonus {\mathbf{\mathtt{B}}}

\def \rbonus {{}^r\mathbf{\mathtt{B}}}
\def \pbonus {{}^p\mathbf{\mathtt{B}}}

\def \lipucbvi { \hyperref[alg:lipucbvi]{{Kernel-UCBVI}}\xspace}
\def \lipucbvigreedy {\hyperref[alg:lipucbvigreedy.maintext]{{Greedy-Kernel-UCBVI}}\xspace}

\def \optimisticQ {\texttt{optimisticQ}}

\def \data {\cD}

\def \kernel {\psi}
\def \kernelfunc {g}

\def \favevent {\cG}

\newcommand{\Lip}[1]{ \mathrm{Lip}\pa{#1} }

\def \bonusvarP { \mathbf{v}_{\mathrm{p}} }

\def \bonusbiasP { \mathbf{b}_{\mathrm{p}} }

\def \bonusvarR { \mathbf{v}_{\mathrm{r}} }

\def \bonusbiasR { \mathbf{b}_{\mathrm{r}} }

\def \gencount {\mathbf{C}}

\def \partitioncount {\mathbf{N}}

\def \tx {\tilde{x}}
\def \ta {\tilde{a}}

\def \weight {w}
\def \normweight {\widetilde{w}}

\def \hoeffdingVarP  {\mathbf{v}_\mathrm{p}}
\def \hoeffdingBiasP {\mathbf{b}_\mathrm{p}}

\def \hoeffdingVarR {\mathbf{v}_\mathrm{r}}
\def \hoeffdingBiasR {\mathbf{b}_\mathrm{r}}

\def \bernvar {\theta_{\mathrm{v}}}
\def \bernbias {\theta_{\mathrm{b}}^1}
\def \bernbiastwo {\theta_{\mathrm{b}}^2}

\def \trueP { P }
\def \estP { \widehat{P} }
\def \trueR { \reward }
\def \estR { \widehat{\reward} }
\def \totalcovdim  {d}
\def \XAcovdim  {d_1}
\def \Xcovdim  {d_2}
\newcommand{\XAcovnumber}[1]{ \cN\pa{ #1, \stateactionspace, \distfunc} }
\newcommand{\Xcovnumber}[1]{ \cN\pa{ #1, \statespace, \Sdistfunc} }

\def \sigmacov {|\cC_\sigma|}
	
\def \Xsigmacov {|\widetilde{\cC}_\sigma|}
\def \logplus {\log^+}

\def \rmd {\mathrm{d}}

\def \termA {\mathbf{(A)}}
\def \termB {\mathbf{(B)}}
\def \termC {\mathbf{(C)}}
\def \termD {\mathbf{(D)}}

\def \mdp {\mathcal{M}}

\def \actionspace {\mathcal{A}}
\def \statespace {\mathcal{X}}
\def \stateactionspace{\statespace\times\actionspace}
\def \transition {P}
\def \reward {r}

\def \policy {\pi}
\def \tV{\widetilde{V}}
\def \tQ{\widetilde{Q}}

\def \eqdef { \overset{\mathrm{def}}{=} }
\newcommand{\pa}[1]{ \left(#1\right) }
\newcommand{\abs}[1]{ \left|#1\right| }
\newcommand{\braces}[1]{ \left\lbrace#1\right\rbrace  }
\newcommand{\sqrbrackets}[1]{\left[ #1 \right]}

\newcommand{\given}{\Big|}

\newcommand{\prob}[2][]{ \mathbb{P}_{#1} \left[ #2 \right] }

\newcommand{\expect}[2][]{ \mathbb{E}_{#1} \left[ #2 \right] }
\newcommand{\variance}[2][]{ \mathbb{V}_{#1} \left[ #2 \right] }

\newcommand{\Exponential}{\cE}

\newcommand{\BigO}[1]{ \mathcal{O}\pa{#1} }
\newcommand{\BigOtilde}[1]{ \widetilde{\mathcal{O}}\pa{#1} }

\newcommand{\indic}[1]{ \mathbb{I}\braces{#1} }
\newcommand{\ceil}[1]{\left\lceil #1 \right\rceil}

\DeclareMathOperator*{\argmin}{argmin}
\DeclareMathOperator*{\argmax}{argmax}
\newcommand{\norm}[1]{ \left\Vert #1 \right\Vert }

\newtheorem{assumption}{Assumption}
\newtheorem{fact}{Fact}
\newtheorem{lemma}{Lemma}
\newtheorem{proposition}{Proposition}
\newtheorem{theorem}{Theorem}
\newtheorem{definition}{Definition}
\newtheorem{corollary}{Corollary}
\newtheorem{remark}{Remark}
\newtheorem{example}{Example}

\usepackage[most]{tcolorbox}
\newtcolorbox{blockquote}{colback=orange!15!white,boxrule=0pt}

\newenvironment{fproposition}
{\begin{blockquote}\begin{proposition}}
		{\end{proposition}\end{blockquote}}
\newenvironment{fcorollary}
{\begin{blockquote}\begin{corollary}}
		{\end{corollary}\end{blockquote}}
\newenvironment{fdefinition}
{\begin{blockquote}\begin{definition}}
		{\end{definition}\end{blockquote}}

\newenvironment{ftheorem}
{\begin{blockquote}\begin{theorem}}
		{\end{theorem}\end{blockquote}}

\newcommand{\cA}{\mathcal{A}}
\newcommand{\cB}{\mathcal{B}}
\newcommand{\cC}{\mathcal{C}}
\newcommand{\cD}{\mathcal{D}}
\newcommand{\cE}{\mathcal{E}}
\newcommand{\cF}{\mathcal{F}}
\newcommand{\cG}{\mathcal{G}}

\newcommand{\cN}{\mathcal{N}}
\newcommand{\cO}{\mathcal{O}}

\newcommand{\cT}{\mathcal{T}}
\newcommand{\cU}{\mathcal{U}}

\newcommand{\cX}{\mathcal{X}}

\newcommand{\bP}{\mathbb{P}}
\newcommand{\NN}{\mathbb{N}}
\newcommand{\RR}{\mathbb{R}}

\newcommand{\ZZ}{\mathbb{Z}}

\newcommand{\tcC}{\widetilde{\cC}}

\newcommand{\tdelta}{\widetilde{\delta}}

\usepackage{xcolor}
\definecolor{Bleu}{RGB}{0,0,204}
\definecolor{Violet}{RGB}{102,0,204}
\definecolor{Rouge}{RGB}{204,0,0}
\definecolor{Highlight}{RGB}{251,0,0}
\definecolor{darkblue}{RGB}{0,0,120}
\definecolor{darkred}{RGB}{140,0,0}
\definecolor{darkgreen}{RGB}{0,120,0}

\usepackage{hyperref}
\hypersetup{
colorlinks,
  citecolor=Bleu,
  linkcolor=darkred,
  urlcolor=Violet}

\icmltitlerunning{Kernel-Based Reinforcement Learning: A Finite-Time Analysis}

\begin{document}

\twocolumn[
\icmltitle{Kernel-Based Reinforcement Learning: A Finite-Time Analysis}



\icmlsetsymbol{equal}{*}

\begin{icmlauthorlist}
	\icmlcorrespondingauthor{Omar D.\,Domingues}{omar.darwiche-domingues@inria.fr}
	\icmlauthor{Omar D. Domingues}{inria,ulille}
	\icmlauthor{Pierre M\'enard}{ovgu}
	\icmlauthor{Matteo Pirotta}{fair}
	\icmlauthor{Emilie Kaufmann}{inria,cnrs}
	\icmlauthor{Michal Valko}{inria,cnrs,deepmind}
\end{icmlauthorlist}

\icmlaffiliation{ovgu}{Otto von Guericke University}
\icmlaffiliation{inria}{Inria Lille}
\icmlaffiliation{fair}{Facebook AI Research, Paris}
\icmlaffiliation{deepmind}{DeepMind Paris}
\icmlaffiliation{cnrs}{CNRS}
\icmlaffiliation{ulille}{Universit\'e de Lille}

\icmlkeywords{exploration, kernel-based, reinforcement learning, regret}

\vskip 0.3in
]



\printAffiliationsAndNotice{}  

\begin{abstract}
We consider the exploration-exploitation dilemma in finite-horizon reinforcement learning problems whose state-action space is endowed with a metric. We introduce Kernel-UCBVI, a model-based optimistic algorithm that leverages the smoothness of the MDP and a non-parametric kernel estimator of the rewards and transitions to efficiently balance exploration and exploitation. For problems with $K$ episodes and horizon $H$, we provide a regret bound of $\widetilde{O}\left( H^3 K^{\frac{2d}{2d+1}}\right)$, where $d$ is the covering dimension of the joint state-action space. This is the first regret bound for kernel-based RL using smoothing kernels, which requires very weak assumptions on the MDP and has been previously applied to a wide range of tasks. We empirically validate our approach in continuous MDPs with sparse rewards.%
\end{abstract}

	\section{Introduction}

	Reinforcement learning (RL) is a learning paradigm in which an agent interacts with an environment by taking actions and receiving rewards. At each time step $t$, the environment is characterized by a state variable $x_t \in \statespace$, which is observed by the agent and influenced by its actions $a_t \in \actionspace$.
    In this work, we consider the online learning problem where the agent has to learn how to act optimally by interacting with an unknown environment.
    To learn efficiently, the agent has to trade-off exploration to gather information about the environment and exploitation to act optimally with respect to the current knowledge.
    The performance of the agent is measured by the \emph{regret}, i.e., the difference between the rewards that would be gathered by an optimal agent and the rewards obtained by the agent.
    This problem has been extensively studied for Markov Decision Processes (MDPs) with finite state-action space.
    \emph{Optimism in the face of uncertainty} (OFU,~\citealt{jaksch2010near}) and \emph{Thompson Sampling}~\citep{strens2000bayesian, osband2013more} principles have been used to design algorithms with sublinear regret.
    However, the guarantees for these approaches cannot be naturally extended to an arbitrarily large state-action space since the regret depends on the number of states and actions.
    When the state-action space is continuous, additional structure in MDP is required to efficiently solve the exploration-exploitation dilemma.

    In this paper, we focus on the online learning problem in MDPs with large or continuous state-action spaces.
    We suppose that the state-action set $\stateactionspace$ is equipped with a known \emph{metric}. For instance, this is typically the case in continuous control problems in which the state space is a subset of $\RR^d$ equipped with the Euclidean metric. 
    As shown by \citet{Ormoneit2002} and \citet{barreto2016practical}, smoothing-kernel approaches converge asymptotically to an optimal policy and perform well empirically in a wide range of continuous MDPs.
   	In this paper, we tackle the problem of \emph{exploration} in such approaches, by proposing an \emph{optimistic} algorithm based on smoothing-kernel estimators of the reward and transition functions of the underlying MDP. The advantages of this approach are:
	 \begin{enumerate*}[label=(\roman*)]
	 	\item it requires weak assumptions on the MDP,
	 	\item it allows us to easily provide expert knowledge to the algorithm through kernel design, and
	 	\item it applies to problems with possibly infinite states without relying on any kind of discretization. 
	 \end{enumerate*}

    \paragraph{Related work} Kernel-based RL (KBRL) using smoothing kernels has been initially proposed by \citet{Ormoneit2002}, who analyzed the algorithm assuming that transitions are generated from \emph{independent} samples, and provide  \emph{asymptotic} convergence guarantees. \citet{barreto2016practical} propose a stochastic factorization technique to reduce the computational complexity of KBRL. 
    In this paper, we provide a modification of KBRL that collects data \emph{online} and for which we prove \emph{finite-time} regret guarantees under weak conditions on the MDP. Under stronger conditions, that use positive-define kernels defining reproducing kernel Hilbert spaces (RKHS) or Gaussian Processes, regret bounds are provided by \citet{pmlr-v89-chowdhury19a}, \citet{chowdhury2020no} and \citet{yang2020provably}.

    Regret minimization in finite MDPs has been extensively studied both in model-based and model-free settings.
    While model-based algorithms~\citep{jaksch2010near, Azar2017, Zanette2019} use the estimated rewards and transitions to perform planning at each episode, model-free algorithms~\citep{Jin2018} directly build an estimate of the optimal Q-function that is updated incrementally.

    For MDPs with continuous state-action space, the sample complexity~\citep{Kakade2003a,Kearns2002near,Lattimore2013general,Pazis13pac} or regret have been studied under structural assumptions.
    Regarding regret minimization,  
    a standard assumption is that rewards and transitions are Lipschitz continuous. \citet{ortner2012online} studied this problem in average reward problems.
    They combined the ideas of UCRL2~\citep{jaksch2010near} and uniform discretization, proving a regret bound of $\BigOtilde{T^{\frac{2d+1}{2d+2}}}$ for a learning horizon $T$ in $d$-dimensional state spaces. This work was later extended by~\citet{Lakshmanan2015} to use a kernel density estimator instead of a frequency estimator for each region of the fixed discretization.
    For each \emph{discrete} region $I(x)$, the density $p(\cdot|I(x), a)$ of the transition kernel is computed through kernel density estimation. The granularity of the discretization is selected in advance based on the properties of the MDP and the learning horizon $T$. As a result, they improve upon the bound of~\citet{ortner2012online}, but require the transition kernels to have densities that are $\kappa$ times differentiable.\footnote{For instance, when $d=1$ and $\kappa\to\infty$, their bound approaches $T^{\frac{2}{3}}$, improving the previous bound of $T^{\frac{3}{4}}$.} However, these two algorithms rely on an intractable optimization problem for finding an optimistic MDP. \citet{QianFPL19} solve this issue by providing an algorithm that uses exploration bonuses, but they still rely on a uniform discretization of the state space.
    \citet{Ok2018exploration} studied the asymptotic regret in Lipschitz MDPs with \emph{finite} state and action spaces, providing a nearly asymptotically optimal algorithm. Their algorithm leverages ideas from asymptotic optimal algorithms in structured bandits~\citep{Combes2017structured} and tabular RL~\citep{burnetas1997optimal}, but does not scale to continuous state-action spaces.
    
    Regarding exploration for finite-horizon MDP with continuous state-action space, \citet{Yang2019b} present an algorithm for deterministic MDPs with Lipschitz transitions.  Assuming that the Q-function is Lipschitz continuous,
    \citet{Song2019} provided a model-free algorithm by combining the ideas of tabular optimistic Q-learning~\citep{Jin2018} with uniform discretization, showing a regret bound of $O(H^{\frac{5}{2}}K^{\frac{d+1}{d+2}})$ where $d$ is the covering dimension of the state-action space. This approach was  extended by~\citet{Sinclair2019} and~\citet{Touati2020zoomRL} to use adaptive partitioning of the state-action space, achieving the same regret bound. \citet{osband2014model} prove a \emph{Bayesian} regret bound in terms of the eluder and Kolmogorov dimension, assuming access to an approximate MDP planner.  
    In addition, there are many results for facing the exploration problem in continuous MDP with \emph{parametric} structure, e.g., linear-quadratic systems~\citep{pmlr-v19-abbasi-yadkori11a} or other linearity assumptions~\citep{Yang2019, Jin2019}, which are outside the scope of our paper.

   \paragraph{Contributions}
   The main contributions of this paper are the following. \textbf{(1)} We provide the first regret bound for KBRL, which applies to a wide range of RL tasks with an entirely \emph{data-dependent} approach;  \textbf{(2)} In order to derive our regret bound, we provide novel concentration inequalities for weighted sums (Lemmas~\ref{lemma:self-normalized-weighted-sum} and \ref{lemma:bernstein-freedman-weighted-sum}) that permit to build confidence intervals for non-parametric kernel estimators (Propositions~\ref{prop:concentration-of-rewards-continuous} and \ref{prop:concentration-of-transitions-continuous}) that are of independent interest. 
   \textbf{(3)} We show that the regret of model-based algorithms, although having a better empirical performance, seem to suffer from a worse dependence on the state-action dimension $d$ than model-free ones. We discuss the origins of this issue by looking at the regret bounds of tabular algorithms.

	\section{Setting}

	\paragraph{Notation}  For any $j \in \ZZ_+$, we define $[j] \eqdef \braces{1, \ldots, j}$. For a measure $P$ and any function $f$, let $Pf \eqdef \int f(y) \mathrm{d}P(y)$.
    If $P(\cdot | x,a)$ is a measure for all $(x,a)$, we let $Pf(x,a) = P(\cdot |x,a) f = \int f(y) \mathrm{d}P(y | x,a)$.

	\paragraph{Markov decision processes} Let $\statespace$ and $\actionspace$ be the sets of states and actions, respectively. We assume that there exists a metric $\rho: (\stateactionspace)^2 \to \RR_{\geq 0}$ on the state-action space and that $(\statespace, \cT_\statespace)$ is a measurable space with $\sigma$-algebra $\cT_\statespace$. We consider an episodic Markov decision process (MDP), defined by the tuple $\mdp \eqdef (\statespace, \actionspace, \horizon, P, r)$ where $\horizon \in \ZZ_+$ is the length of each episode, $P = \braces{P_h}_{h\in[H]}$ is a set of transition kernels from $(\stateactionspace)\times \cT_\statespace$ to $\RR_{\geq 0}$, and $r = \braces{r_h}_{h\in[H]}$ is a set of reward functions from $\stateactionspace$ to $[0, 1]$. A  policy $\pi$ is a mapping from $[H] \times \statespace$ to $\actionspace$, such that $\pi(h, x)$ is the action chosen by $\pi$ in state $x$ at step $h$. The Q-value of a policy $\pi$ for state-action $(x,a)$ at step $h$ is the expected sum of rewards obtained by taking action $a$ in state $x$ at step $h$ and then following the policy $\pi$, that is
	\begin{align*}
	Q_h^\pi(x,a) & \eqdef  \expect{\sum_{h'=h}^H \reward_{h'}(x_{h'}, a_{h'}) \given x_h = x,a_h=a},
	\end{align*}
	where the expectation is under transitions in the MDP $x_{h'+1} \sim P_{h'}(\cdot | x_{h'},a_{h'})$ and $a_{h'}=\pi(h',x_{h'})$. 
	The value function of policy $\pi$ at step $h$ is $V_h^\pi(x) = Q_h^\pi(x,\pi(h,x))$. 
	The optimal value functions, defined by $V_h^*(x) \eqdef \sup_{\pi} V_h^\pi(x)$ for $h \in [H]$, satisfy the optimal Bellman equations \citep{Puterman1994}: $V_h^*(x) = \max_{a \in \actionspace} Q_h^*(x, a)$, where 
	\begin{align*}
	 Q_h^*(x, a) \eqdef \reward_h(x,a) + \int_{\statespace} V_{h+1}^*(y)\mathrm{d}P_h(y|x,a)
	\end{align*}
	and, by definition, $V_{H+1}^*(x) = 0$ for all $x \in \statespace$.

    \paragraph{Learning problem}
    A reinforcerment learning agent interacts with $\mdp$ in a sequence of episodes $k \in [K]$ of fixed length $H$ by playing a policy $\pi_k$ in each episode, where
    the initial state $x_{1}^{k}$ is chosen arbitrarily and revealed to the agent.
    The learning agent does not know $P$ and $r$ and it selects the policy $\pi_k$ based on the samples observed over previous episodes.
    Its performance is measured by the regret
	$\regret(K) \eqdef \sum_{k=1}^K \pa{V_1^*(x_1^k) - V_1^{\policy_k}(x_1^k)}$.

	We make the following assumptions:

	\begin{assumption}
		\label{assumption:metric-state-space}
		The metric $\distfunc$ is given to the learner. Also, there exists a metric $\Sdistfunc$ on $\statespace$ and a metric $\Adistfunc$ on $\actionspace$ such that, for all $(x, x',a, a')$,
		$
		 \dist{(x, a), (x', a')} = \Sdist{x, x'} + \Adist{a, a'}.
		$
	\end{assumption}

	\begin{assumption}
		\label{assumption:lipschitz-rewards-and-transitions}
		The reward functions are $\lambda_r$-Lipschitz and the transition kernels are $\lambda_p$-Lipschitz with respect to the 1-Wasserstein distance: $\forall (x, a, x', a')$ and $\forall h \in [H]$,
		$$
		\abs{r_h(x,a) - r_h(x',a')} \leq \lambda_r \dist{(x,a), (x',a')}, 
		\quad \text{and}
		$$
		$$
		\Wassdist{P_h(\cdot|x, a), P_h(\cdot|x', a')} \leq \lambda_p \dist{(x,a), (x',a')}
		$$ where, for two measures $\mu$ and $\nu$, we have
		$
		\Wassdist{\mu, \nu} \eqdef \sup_{f: \mathrm{Lip}(f) \leq 1}  \int_{\statespace} f(y)(\mathrm{d}\mu(y)-\mathrm{d}\nu(y))
		$
		and where, for a function $f:\statespace\to\RR$, $\mathrm{Lip}(f)$ denotes its Lipschitz constant \wrt $\Sdistfunc$.
	\end{assumption}

To assess the relevance of these assumptions, we show below that they apply to deterministic MDPs with Lipschitz reward and transition functions (whose transition kernels are \emph{not} Lipschitz \wrt the total variation distance).   

\begin{example}[Deterministic MDP in $\RR^d$]
	Consider an MDP $\mdp$ with a finite action set, with a compact state space $\statespace \subset \RR^d$, and deterministic transitions $y = f(x, a)$, \ie $P_h(y|x, a) = \delta_{f(x,a)}(y)$. Let $\Sdistfunc$ be the Euclidean distance on $\RR^d$ and $\Adist{a, a'} = 0$ if $a = a'$ and $\infty$ otherwise. Then, if for all $a\in\actionspace$, $x\mapsto r_h(x, a)$ and $x\mapsto f(x, a)$ are Lipschitz, $\mdp$ satisfies assumptions \ref{assumption:metric-state-space} and \ref{assumption:lipschitz-rewards-and-transitions}.
\end{example}

Under our assumptions, the optimal $Q$ functions are Lipschitz continuous:

	\begin{lemma}
		\label{lemma:q_function_is_lipschitz}
		Let $L_h \eqdef \sum_{h'= h}^H \lambda_r \lambda_p^{H-h'}$. Under Assumption \ref{assumption:lipschitz-rewards-and-transitions}, for all $(x, a, x', a')$ and for all $h \in [H]$, we have $\abs{Q_h^*(x, a) - Q_h^*(x', a')} \leq L_h \dist{(x,a), (x', a')}$, \ie the optimal $Q$-functions are Lipschitz continuous.
	\end{lemma}

	\section{Algorithm}
    \label{sec:algorithm}
    
	  In this section, we present \lipucbvi, a model-based algorithm for exploration in MDPs in metric spaces that employs \emph{kernel smoothing} to estimate the rewards and transitions, for which we derive confidence intervals. \lipucbvi uses exploration bonuses based on these confidence intervals to efficiently balance exploration and exploitation. Our algorithm requires the knowledge of the metric $\distfunc$ on $\stateactionspace$ and of the Lipschitz constants of the rewards and transitions.\footnote{This assumption is standard in previous works in RL~\citep[\eg][]{ortner2012online,Sinclair2019}. Theoretically, we could replace the Lipschitz constant $L_1$ by $\log k$, in each episode $k$, and our regret bound would have an additive term of order $H e^{L_1}$, since $Q_h^k$ would be optimistic for $\log k \geq L_1$~\citep[see \eg][]{ReeveM018}. However, this would degrade the performance of the algorithm in practice.}

	\subsection{Kernel Function}

	 We leverage the knowledge of the state-action space metric to define the kernel function.
	Let $u, v \in \stateactionspace$. For some function $\kernelfunc: \RR_{\geq 0} \to [0, 1]$, we define the kernel function as
	$$
	\kernel_{\sigma}(u, v) \eqdef \kernelfunc\pa{\dist{u, v}/\sigma}
	$$
	where $\sigma$ is the bandwidth parameter that controls the degree of ``smoothing'' of the kernel.
	In order to be able to construct valid confidence intervals, we require certain structural properties for $\kernelfunc$.

	\begin{assumption}
		\label{assumption:kernel-behaves-as-gaussian}
		The function $\kernelfunc: \RR_{\geq 0} \to [0, 1]$ is differentiable, non-increasing, $g(4) > 0$, and there exists two constants $C_1^g, C_2^g > 0$ that depend only on $\kernelfunc$ such that 
		\begin{align*}
			\kernelfunc(z) \leq C_1^g\exp(-z^2/2) \mbox{ and } \sup_{z}\abs{\kernelfunc'(z)} \leq  C_2^g.
		\end{align*}
	\end{assumption}
	
	This assumption is trivially verified by the Gaussian kernel $\kernelfunc(z) = \exp(-z^2/2)$. Other examples include the kernels $\kernelfunc(z) = \exp(-|z|^p/2)$ for $p > 2$.

\begin{algorithm}[t]
	\caption{\lipucbvi}
	\label{alg:lipucbvi}
	\begin{small}
		\begin{algorithmic}
			\STATE {\bfseries Input:} global parameters $K, H, \delta, \lambda_r, \lambda_p, \sigma, \beta$
			\STATE initialize data lists $\data_h = \emptyset$ for all $h \in [H]$
			\FOR{episode $k = 1, \ldots, K$}
			\STATE get initial state $x_1^k$
			\STATE $Q_h^k = \optimisticQ(k, \braces{\data_h}_{h\in[H]})$
			\FOR{step $h=1, \ldots, H$}
			\STATE execute $a_h^k = \argmax_a Q_h^k(x_h^k, a)$
			\STATE observe reward $r_h^k$ and next state $x_{h+1}^k$
			\STATE add sample $(x_h^k, a_h^k, x_{h+1}^k, r_h^k)$ to $\data_h$
			\ENDFOR
			\ENDFOR
		\end{algorithmic}
	\end{small}
\end{algorithm}

 	\begin{algorithm}[t]
 	\caption{\optimisticQ}
 	\label{alg:optimisticQ}
 	\begin{small}
 		\begin{algorithmic}
 			\STATE {\bfseries Input:} episode $k$, data $\braces{\data_h}_{h\in[H]}$
 			\STATE Initialize $V_{H+1}^k(x) = 0$ for all $x$
 			\FOR{step $h=H, \ldots, 1$}
 			\STATE \textcolor{darkgreen}{// Compute optimistic targets}
 			\FOR{$m = 1, \ldots, k-1$}
 			\STATE $\widetilde{Q}_h^k(x_h^m, a_h^m) = \sum_{s=1}^{k-1} \widetilde{w}_h^s(x_h^m, a_h^m) \pa{r_h^s + V_{h+1}^k(x_{h+1}^s)}$
 			\STATE $\widetilde{Q}_h^k(x_h^m, a_h^m) =\widetilde{Q}_h^k(x_h^m, a_h^m) + \bonus_h^k(x_h^m, a_h^m)$
 			\ENDFOR
 			\STATE \textcolor{darkgreen}{// Interpolate the Q function}
 			\STATE $Q_h^k(x,a) = \min\limits_{s \in [k-1]} \left (
 			\widetilde{Q}_h^k(x_h^{s},a_h^{s}) + L_h\dist{(x,a), (x_h^{s},a_h^{s})}
 			\right)$
 			\FOR{$m = 1, \ldots, k-1$}
 			\STATE$V_h^k(x_h^{m}) = \min\pa{H-h+1, \max_{a\in\actionspace}Q_h^k(x_h^{m},a)}$
 			\ENDFOR
 			\ENDFOR
 			\STATE {\bfseries return} $Q_h^k$
 		\end{algorithmic}
 	\end{small}
 \end{algorithm}

	\subsection{Kernel Estimators and Optimism}
    In each episode $k$, \lipucbvi computes an optimistic estimate $Q_h^k$ for all $h$, which is an upper confidence bound on the optimal $Q$ function $Q^*_h$, and plays the associated greedy policy.
    Let $(x_h^s, a_h^s, x_{h+1}^s, r_h^s)$ be the random variables representing the state, the action, the next state and the reward at step $h$ of episode $s$, respectively.
    We denote by $\data_h = \braces{(x_h^s, a_h^s, x_{h+1}^s, r_h^s)}_{s \in [k-1]}$ for $h \in [H]$ the samples collected at step $h$ before episode $k$. 
    
    For any $(x, a)$ and $(s, h) \in [K]\times[H]$, we define the \emph{weights} and the \emph{normalized weights} as
    \begin{align*}
    & w_h^s(x,a) \eqdef  \kernel_{\sigma}((x,a),(x_h^s, a_h^s)) 
    \quad\text{and }\quad \\
    & \widetilde{w}_h^s(x,a) \eqdef \frac{ w_h^s(x,a) }{ \beta +  \sum_{l=1}^{k-1}w_h^l(x,a) } ,
    \end{align*}
    where $\beta >0$ is a regularization term.
    These weights are used to compute an estimate of the rewards and transitions for each state-action pair\footnote{Here, $\dirac{x}$ denotes the Dirac measure with mass at $x$.}:
	\begin{align*}
            & \widehat{\reward}_h^k(x, a)   \eqdef  \sum_{s=1}^{k-1}\widetilde{w}_h^s(x,a) \reward_h^s,
            \quad \\ 
            & \widehat{\transition}_h^k(y|x, a)\eqdef  \sum_{s=1}^{k-1} \widetilde{w}_h^s(x,a) \dirac{x_{h+1}^s}(y).
	\end{align*}
    As other algorithms using OFU, \lipucbvi computes an optimistic Q-function $\widetilde{Q}_h^k$ through value iteration, a.k.a.\ backward induction:
    \begin{align}
        \label{eq:optimistic.kernelbellman}
        \widetilde{Q}_h^k(x, a) = \widehat{\reward}_h^k(x, a) + \widehat{P}_h^k V_{h+1}^k(x, a) + \bonus_h^k(x, a)
    \end{align}
    where $V_{H+1}^k(x) =0$ for all $x \in \statespace$ and $\bonus_h^k(x, a)$ is an exploration bonus described later.      From Lemma~\ref{lemma:q_function_is_lipschitz}, the true $Q$ function $Q_h^*$ is $L_h$-Lipschitz. Computing $\widetilde{Q}_h^k$ for all previously visited state action pairs $(x_h^s, a_h^s)$ for $s \in [k-1]$  permits to define a \emph{$L_h$-Lipschitz upper confidence bound} and the associated value function:
    \begin{align*}
    	 Q_h^k(x,a) \eqdef \min_{s \in [k-1]} \pa{
    	\widetilde{Q}_h^k(x_h^{s},a_h^{s}) + L_h\dist{(x,a), (x_h^{s},a_h^{s})}}
    \end{align*}
    and $ V_h^k(x) \eqdef \min\pa{H-h+1, \max_{a'} Q_h^k(x, a')}$. The policy $\pi_k$ executed by \lipucbvi is the greedy policy with respect to $Q_h^k$ (see Alg.~\ref{alg:lipucbvi}).
     
    Let $\gencount_h^k(x, a) \eqdef \beta + \sum_{s=1}^{k-1} w_h^s(x, a)$ be the \emph{generalized counts}, which are a proxy for the number of visits to $(x, a)$. The exploration bonus is defined based on the uncertainties on the transition and reward estimates and takes the form
	\begin{align*}
	\bonus_h^k(x, a) \approx
	\frac{H}{\sqrt{\gencount_h^k(x, a)}} + \frac{\beta H}{\gencount_h^k(x, a)} + L_1\sigma 
	\end{align*} 
	where we omit constants and logarithmic terms. Refer to Eq.~\ref{eq:bonus_full} in App.~\ref{app:algorithm} for the exact definition.


	\section{Theoretical Guarantees \& Discussion}

	The theorem below gives a high probability regret bound for $\lipucbvi$. It features the  $\sigma$-covering number of the state-action space. The $\sigma$-covering number of a metric space, formally defined in Def.\,\ref{def:covering} (App.~\ref{app:preliminaries}), is roughly the number of $\sigma$-radius balls required to cover the entire space. The covering dimension of a space is the smallest number $d$ such that its $\sigma$-covering number is $\BigO{\sigma^{-d}}$. For instance, the covering number of a ball in $\RR^d$ with the Euclidean distance is $\BigO{\sigma^{-d}}$ and its covering dimension is $d$.
		
	\begin{theorem}
		\label{theorem:regret-bound-restated}
		With probability at least $1-\delta$, the regret of $\lipucbvi$ for a bandwidth $\sigma$ satisfies
		\begin{align*}
		\regret(K) \leq \BigOtilde{H^2 \sqrt{\abs{\cC_\sigma}K} +  L_1 K H \sigma + H^3 \sigmacov\Xsigmacov  },  
		\end{align*}
		where $\sigmacov$ and $\Xsigmacov$ are the $\sigma$-covering numbers of $(\stateactionspace,\distfunc)$ and $(\statespace, \Sdistfunc)$, respectively, and $L_1$ is the Lipschitz constant of the optimal $Q$-functions. 
	\end{theorem}
	\begin{proof}
		Restatement of Theorem \ref{theorem:regret-bound} in App.\,\ref{app:optimism-regret}. A proof sketch is given in Section \ref{app:proof_sketch}.
	\end{proof}
	\begin{corollary}
		\label{corollary:main-text-corollary}
		By taking $\sigma = (1/K)^{1/(2d+1)}$, $\regret(K) = \BigOtilde{ H^3 K^{\max\pa{\frac{1}{2}, \frac{2d}{2d+1}}}}$, where $d$ is the covering dimension of the state-action space.
	\end{corollary}

    \begin{remark}
    	\label{remark:stationary}
     As for other model-based algorithms, the dependence on $H$ can be improved if the transitions are stationary, \ie do not depend on $h$. In this case, the regret of  
	\lipucbvi becomes $\BigOtilde{H^2K^{\frac{2d}{2d+1}}}$ due to a gain a factor of $H$ in the second order term (see App.~\ref{app:variants}). 
	\end{remark}

	To the best of our knowledge, this is the first regret bound for kernel-based RL using smoothing kernels, and we present below further discussions on this result.

	\paragraph{Comparison to lower bound for Lipschitz MDPs} In terms of the number of episodes $K$ and the dimension $d$, the lower bound for Lipschitz MDPs is $\Omega(K^{(d+1)/(d+2)})$, which is a consequence of the result for contextual Lipschitz bandits~\citep{Slivkins2014contextual}. In terms of $H$, the optimal dependence can be conjectured to be $H^{3/2}$, which is the case for tabular MDPs \cite{Jin2018}.\footnote{See also \citep[][Sec. 4.4]{Sinclair2019}.} For $d=1$, our bound has an optimal dependence on $K$, leading to a regret of order $\BigOtilde{H^3K^{2/3}}$, or $\BigOtilde{H^2 K^{2/3}}$ when the transitions are stationary (see Remark~\ref{remark:stationary}). 
	

	\paragraph{Comparison to other upper bounds for Lipschitz MDPs} The best available upper bound in this setting, in terms of $K$ and $d$, is $\cO\pa{H^{5/2} K^{\frac{d+1}{d+2}}}$, which is achieved by model-free algorithms performing either uniform or adaptive discretization of the state-action space \cite{Song2019, Sinclair2019, Touati2020zoomRL}.

	\paragraph{Relevance of a kernel-based algorithm} Although our upper bound does not match the lower bound for Lipschitz MDPs, kernel-based RL can be a very useful tool in practice to handle the bias-variance trade-off in RL. It allows us to easily provide expert knowledge to the algorithm through kernel design, which can be seen as introducing more bias to reduce the variance of the algorithm and, consequently, improve the learning speed. As shown by \citet{kveton2012kernel} and \citet{barreto2016practical},  KBRL are empirically successful in medium-scale tasks ($d\approx 10$), such as control problems, HIV drug scheduling and an epilepsy suppression task. In such problems, \lipucbvi can be used to enhance exploration, and the confidence intervals we derive here may also be useful in settings such as robust planning \citep{lim19robust}. Interestingly, \citet{badia2020never} have shown that kernel-based exploration bonuses similar to the ones derived in this paper can improve exploration in Atari games.

	\paragraph{Regularity assumptions} The regret bound we provide only requires only weak assumptions on the MDP: we assume that both the transitions and rewards are Lipschitz continuous, but we have no constraints on the behavior of the Bellman operator. As a consequence, the regret bounds suffer from the curse of dimensionality: as $d$ goes to infinity, both the lower and upper bounds become linear in the number of episodes $K$. 
	Other settings, such as low-rank MDPs \citep{Jin2019} and RKHS approximations \citep{yang2020provably, chowdhury2020no} achieve regret bounds scaling with $\sqrt{K}$, but they require much stronger assumptions on the MDP, such as the closedness of the Bellman operator in the function class used to represent $Q$ functions, which is a condition that is much harder to verify. \citet{barreto2016practical} show that KBRL (with smoothing kernels) can be related to low-rank MDPs, and we believe that our analysis brings new elements to study this trade-off that exists between regularity assumptions and regret bounds.

	\paragraph{Model-free vs. Model-based}
	An interesting remark comes from the comparison between our algorithm and recent model-free approaches in continuous MDPs~\citep{Song2019,Sinclair2019,Touati2020zoomRL}.
	These algorithms are based on optimistic Q-learning~\citep{Jin2018}, to which we refer as OptQL, and achieve a regret of order $\BigOtilde{H^{\frac{5}{2}}K^{\frac{d+1}{d+2}}}$, which has an optimal dependence on $K$ and $d$. While we achieve the same $\BigOtilde{K^{2/3}}$ regret when $d=1$, our bound is slightly worse for $d>1$. To understand this gap, it is enlightening to look at the regret bound for tabular MDPs.

	Since our algorithm is inspired by UCBVI \citep{Azar2017} with Chernoff-Hoeffding bonus, we compare it to OptQL, which is used by \citep{Song2019,Sinclair2019,Touati2020zoomRL}, with the same kind of exploration bonus. Consider an MDP with $X$ states and $A$ actions and non-stationary transitions. 
	UCBVI has a regret bound of $\BigOtilde{H^2\sqrt{XAK} + H^3X^2A}$ while OptQL has $\BigOtilde{H^{5/2}\sqrt{XAK} + H^2XA}$.
    As we can see, OptQL is a $\sqrt{H}$-factor worse than UCBVI when comparing the first-order term, but it is $HX$ times better in the second-order term. For large values of $K$, second-order terms can be neglected in the comparison of the algorithms in tabular MDPs, since they do not depend on $K$. However, they play an important role in continuous MDPs, where $X$ and $A$ are replaced by the $\sigma$-covering number of the state-action space, which is roughly $1/\sigma^d$.
     In tabular MDPs, the second-order term is constant (\ie does not depend on $K$).
    On the other hand, in continuous MDPs, the algorithms define the granularity of the representation of the state-action space based on the number of episodes, connecting the number of states $X$ with $K$. For example, in~\citep{Song2019} the $\epsilon$-net used by the algorithm is tuned such that $\epsilon = (HK)^{-1/(d+2)}$ (see also~\citealt{ortner2012online,Lakshmanan2015,QianFPL19}). Similarly, in our algorithm we have that $\sigma = K^{-1/(2d+1)}$.
    For this reason, the second-order term in UCBVI becomes the dominant term in our analysis, leading to a worse dependence on $d$ compared to model-free algorithms, as highlighted in the proof sketch (Sec.~\ref{app:proof_sketch}).
    For similar reasons, \lipucbvi has an additional $\sqrt{H}$ factor compared to model-free algorithms based on \citep{Jin2018}.
    This shows that the direction of achieving first-order optimal terms at the expense of higher second-order terms may not be justified outside the tabular case. Whether this is a flaw in the algorithm design or in the analysis is left as an open question. However, as observed in Section~\ref{sec:experiments}, model-based algorithms seem to enjoy a better empirical performance.

    \section{Improving the Computational Complexity}

    \lipucbvi is a non-parametric model-based algorithm and, consequently, it inherits the weaknesses of these approaches.
     In order to be data adaptive, it needs to store all the samples $(x_h^k, a_h^k, x_{h+1}^k, r_h^k)$ and their optimistic values $\widetilde{Q}_h^k$ and $V_h^k$ for $(k, h) \in [K]\times[H]$, leading to a total memory complexity of $\BigO{HK}$.
    Like standard model-based algorithms, it needs to perform planning at each episode which gives a total runtime of $\BigO{HAK^3}$\footnote{Since the runtime of an episode $k$ is $\BigO{HAk^2}$.}, where the factor $A$ takes into account the complexity of computing the maximum over actions.\footnote{While in theory the algorithm works with a compact action space, the main practical issue resides in the computation of the best action (\ie $a_h^k = \argmax_a Q_h^k(x_h^k, a)$). In this case, we could resort to black box optimization algorithms~\citep[\eg][Sec. 3.3]{Munos2014from}, which might require the discretization of the action space. This is however less critical than the discretization of the state-space, since the possible actions must always be known in advance, unlike the set of possible states.}
    \lipucbvi has similar time and space complexity of recent approaches for low-rank MDPs~\citep{Jin2019,zanette2019tslinear}.

    To alleviate the computational burden of \lipucbvi, we leverage Real-Time Dynamic Programming (RTDP), see \citep{barto1995learning}, to perform incremental planning.
    Similarly to OptQL, RTDP-like algorithms maintain an optimistic estimate of the optimal value function that is updated incrementally by interacting with the MDP. The main difference is that the update is done by using an estimate of the MDP (\ie model-based) rather than the observed transition sample.
    In episode $k$ and step $h$, our algorithm, named \lipucbvigreedy, computes an upper bound $\widetilde{Q}_h^k(x_h^k,a)$ for each action $a$ using the kernel estimate as in Eq.~\eqref{eq:optimistic.kernelbellman}.
    Then, it executes the greedy action $a_h^k = \argmax_{a \in \mathcal{A}} \widetilde{Q}_h^k(x_h^k,a)$.
    As a next step, it computes $\tV_h^k(x_h^k) = \widetilde{Q}_h^k(x_h^k,a_h^k)$ and refines the previous $L_h$-Lipschitz upper confidence bound on the value function
    \[V_h^{k+1}(x) =\min\pa{V_h^{k}(x), \tV_h^k(x_h^k) + L_h\Sdist{x,x_h^k}}.\]

    The complete description of \lipucbvigreedy is given in Alg.~\ref{alg:lipucbvigreedy.maintext} in App.~\ref{app:efficient_implementation}. The total runtime of this efficient version is $\BigO{HAK^2}$ with total memory complexity of $\BigO{HK}$.

    RTDP has been recently analyzed by~\cite{efroni2019tight} in tabular MDPs.
    Following their analysis, we prove the following theorem, which shows that \lipucbvigreedy achieves the same guarantees of \lipucbvi with a large improvement in computational complexity.

    \begin{theorem}
    	\label{theorem:regret-greedy-main-text}
    With probability at least $1-\delta$, the regret of $\lipucbvigreedy$ for a bandwidth $\sigma$ is of order
      $ \regret(K) =  \BigOtilde{\regret(K, \lipucbvi)+  H^2\Xsigmacov}$,
    where $\Xsigmacov$ is the $\sigma$-covering number of state space. This results in a regret of $\BigOtilde{H^3 K^{2d/(2d+1}}$ when $\sigma=(1/K)^{1/(2d+1)}$.
    \end{theorem}
    \emph{Proof.}
    The complete proof is provided in App.~\ref{app:efficient_implementation}.
    The key properties for proving this regret bound are: \emph{(i)} optimism, and \emph{(ii)} the fact that $(V_h^k)$ are point-wise non-increasing. 

	Besides RTDP, other techniques previously proposed to accelerate KBRL can also be applied, notably the use of \emph{representative states} \cite{kveton2012kernel, barreto2016practical} that merge states that are close to each other to avoid a per-episode runtime that increases with $k$.

\section{Proof sketch}
	\label{app:proof_sketch}
	
	We now provide a sketch of the proof of our main result, Theorem \ref{theorem:regret-bound-restated}. The complete proof is given in the Appendix. The analysis splits into three parts: (i) deriving confidence intervals for the reward and transition kernel estimators; (ii) proving that the algorithm is optimistic, i.e., that $V_h^k(x) \geq V_h^*(x)$ for any $(x, k, h)$ on a high probability event $\favevent$; and (iii) proving an upper bound on the regret by using the fact that $\regret(K) = \sum_k \pa{V_1^*(x_1^k) - V_1^{\policy_k}(x_1^k)} \leq \sum_k \pa{V_1^k(x_1^k) - V_1^{\policy_k}(x_1^k)}$.
	
	\subsection{Concentration}
	
	The most interesting part is the concentration of the transition kernel.
	Since $\widehat{P}_h^k(\cdot|x, a)$ are  weighted sums of Dirac measures, we cannot bound the distance between $P_h(\cdot|x, a)$ and $\widehat{P}_h^k(\cdot|x, a)$ directly. Instead, for $V_{h+1}^*$ the optimal value function at step $h+1$, we bound the difference
	{\small
	\begin{align*}
	& \abs{(\widehat{P}_h^k-P_h)V_{h+1}^*(x, a)} 
	\\
	& 
	= \abs{ \sum_{s=1}^{k-1}\widetilde{w}_h^s(x,a)V_{h+1}^*(x_{h+1}^s) - P_hV_{h+1}^*(x, a)} \\
	& \leq \underbrace{ \abs{ \sum_{s=1}^{k-1}\widetilde{w}_h^s(x,a) \pa{ V_{h+1}^*(x_{h+1}^s) - P_hV_{h+1}^*(x_h^s, a_h^s) } }}_{\mathbf{(A)}}  \\
	& 
	 + \underbrace{\lambda_p L_{h+1}\sum_{s=1}^{k-1}\widetilde{w}_h^s(x,a)\dist{(x,a), (x_h^s,a_h^s)}}_{\mathbf{(B)}}
	+ \underbrace{\frac{ \beta \norm{V_{h+1}^*}_\infty }{ \gencount_h^k(x,a) }}_{\mathbf{(C)}}\,\cdot	
	\end{align*}
	}
	The term $\mathbf{(A)}$ is a weighted sum of a martingale difference sequence. To control it, we propose a new Hoeffding-type inequality, Lemma \ref{lemma:self-normalized-weighted-sum}, that applies to weighted sums with random weights. The term $\mathbf{(B)}$ is a bias term that is obtained using the fact that $V_{h+1}^*$ is $L_{h+1}$-Lipschitz and that the transition kernel is $\lambda_p$-Lipschitz, and can be shown to be proportional to the bandwidth $\sigma$ under Assumption \ref{assumption:kernel-behaves-as-gaussian} (Lemma \ref{lemma:kernel-bias}). The term $\mathbf{(C)}$ is the bias introduced by the regularization parameter $\beta$. Hence, for a fixed state-action pair $(x, a)$, we show that\footnote{Here, $\lesssim$ means smaller than or equal up to logarithmic terms.}, with high-probability,
	\begin{align*}
	\abs{(\widehat{P}_h^k-P_h)V_{h+1}^*(x, a)} \lesssim  \frac{H}{\sqrt{\gencount_h^k(x, a)}} + \frac{\beta H}{\gencount_h^k(x, a)} + L_1\sigma.
	\end{align*}
	Then, we extend this bound to all $(x, a)$ by leveraging the continuity of all the terms involving $(x, a)$ and a covering argument. This continuity is a consequence of kernel smoothing, and it is a key point in avoiding a discretization of $\stateactionspace$ in the algorithm.
	
	In Theorem~\ref{theorem:favorable-event}, we define a favorable event $\cG$, of probability larger than $1-\delta/2$, in which (a more precise version of) the above inequality holds, the mean rewards belong to their confidence intervals, and we further control the deviations of $(\widehat{P}_h^k-P_h)f(x, a)$  for \emph{any} $2L_1$-Lipschitz function $f$. This last part is obtained thanks to a \emph{new Bernstein-like concentration inequality} for weighted sums (Lemma~\ref{lemma:bernstein-freedman-weighted-sum}).

	\subsection{Optimism}
	To prove that the optimistic value function $V_h^k$ is indeed an upper bound on $V_h^*$, we proceed by induction on $h$ and we use the $Q$ functions. When $h = H+1$, we have $Q_{H+1}^k(x, a) = Q_{H+1}^*(x, a) = 0$ for all $(x, a)$, by definition. Assuming that $Q_{h+1}^k(x, a) \geq Q_{h+1}^*(x, a)$ for all $(x, a)$, we have $V_{h+1}^k(x) \geq V_{h+1}^*(x)$ for all $x$. Then, the bonuses are defined so that $\widetilde{Q}_h^k(x,a) \geq Q_h^*(x,a)$ for all $(x, a)$, on the event $\favevent$.

	In particular $\widetilde{Q}_h^k(x_h^s,a_h^s) - Q_h^*(x_h^s,a_h^s) \geq 0$ for all $s \in [k-1]$, which gives us%
	\begin{align*}
	& \widetilde{Q}_h^k(x_h^{s},a_h^{s}) + L_h\dist{(x,a), (x_h^{s},a_h^{s})} 
	\\
	& 
	 \geq Q_h^*(x_h^{s},a_h^{s}) + L_h\dist{(x,a), (x_h^{s},a_h^{s})} \geq Q_h^*(x,a)
	\end{align*}
	for all $s \in [k-1]$, since $Q_h^*$ is $L_h$-Lipschitz. It follows from the definition of $Q_h^k$ that $
	Q_h^k(x,a) \geq Q_h^*(s,a)$, which in turn implies that, for all $x$, $V_h^k(x) \geq V_h^*(x)$ in $\favevent$.
	
	\subsection{Bounding the regret}
	\label{subsec:boundingregret}
	
	To provide an upper bound on the regret in the event $\favevent$, let $\delta_h^k \eqdef V_h^k(x_h^k) - V_h^{\pi_k}(x_h^k)$. The fact that $V_h^k \geq V_h^*$ gives us $\regret(K) \leq \sum_k \delta_1^k$.
	Introducing $(\tilde{x}_h^k,\tilde{a}_h^k)$, the state-action pair in the past data $\data_h$ that is the closest to $(x_h^k,a_h^k)$ and letting $\square_h^k = \dist{(\tilde{x}_h^k,\tilde{a}_h^k), (x_h^k,a_h^k)}$,
	we bound $\delta_h^k$ using the following decomposition:
	{\small
	\begin{align*}
	\delta_h^k 
	& 
	\leq Q_h^k(x_h^k,a_h^k) - Q_h^{\pi_k}(x_h^k,a_h^k) 
	\\
	& 
	 \leq  \widetilde{Q}_h^k(\tilde{x}_h^k,\tilde{a}_h^k) - Q_h^{\pi_k}(x_h^k,a_h^k) + L_h \square_h^k 
	\\
	& 
	\leq 2\,\bonus_h^k(\tilde{x}_h^k,\tilde{a}_h^k) + (L_h + \lambda_p L_h +\lambda_r) \square_h^k  
	\\
	& 
	  + \pa{\widehat{P}_h^k - P_h} V_{h+1}^*(\tilde{x}_h^k, \tilde{a}_h^k)
  	    \quad\quad\quad\quad\quad\;\, \mathbf{(1)}
	 \\
	 & 
	  + P_h\pa{V_{h+1}^k-V_{h+1}^{\pi_k}}(x_h^k,a_h^k) 
	  	 \quad\quad\quad\quad\;\;\, \mathbf{(2)}
	 \\
	 & 
	   + \pa{\widehat{P}_h^k - P_h}\pa{V_{h+1}^k-V_{h+1}^*}(\tilde{x}_h^k, \tilde{a}_h^k)
	   \quad \mathbf{(3)}
   	.
	\end{align*}
	}
	The term $\mathbf{(1)}$ is shown to be smaller than $\bonus_h^k(\tilde{x}_h^k,\tilde{a}_h^k)$, by definition of the bonus. The term $\mathbf{(2)}$ can be rewritten as $\delta_{h+1}^k$ plus a martingale difference sequence $\xi_{h+1}^k$. To bound the term $\mathbf{(3)}$, we use that $V_{h+1}^k-V_{h+1}^*$ is $2L_1$-Lipschitz. The uniform deviations that hold on event $\cG$ yield
	\begin{align*}
	\text{\ding{194}} \lesssim \frac{1}{H}\pa{\delta_{h+1}^k + \xi_{h+1}^k} + \frac{ H^2 \Xsigmacov}{\gencount_h^k(\tilde{x}_h^k,\tilde{a}_h^k)}  + L_1\square_h^k + L_1 \sigma\,.
	\end{align*}
	
	When $\square_h^k > 2\sigma$, we bound $\delta_h^k$ by $H$ and we verify that $H \sum_{h=1}^H\sum_{k=1}^K \indic{\square_h^k > 2\sigma} \leq H^2 \sigmacov$ by a pigeonhole argument. Hence, we can focus on the case where $\square_h^k \leq 2\sigma$, and add $H^2 \sigmacov$ to the regret bound, to take into account the steps $(k, h)$ where $\square_h^k > 2\sigma$. The sum of $\xi_{h+1}^k$ over $(k, h)$ is bounded by $\BigOtilde{H^\frac{3}{2}\sqrt{K}}$ by Hoeffding-Azuma's inequality, on some event $\cF$ of probability larger than $1-\delta/2$.
	Now, we focus on the case where $\square_h^k \leq 2\sigma$ and we omit the terms involving $\xi_{h+1}^k$. Using the definition of the bonus, we obtain
	{\small
	\begin{align*}
	\delta_h^k \lesssim \pa{1+\frac{1}{H}}\delta_{h+1}^k + \frac{H}{\sqrt{\gencount_h^k(\tilde{x}_h^k,\tilde{a}_h^k)}}  + \frac{H^2\Xsigmacov}{\gencount_h^k(\tilde{x}_h^k,\tilde{a}_h^k)}  + L_1 \sigma.
	\end{align*}
	}
	Using the fact that $(1+1/H)^{H} \leq e$, we have, on $\cG\cap \cF$,
	{\small
	\begin{align*}
	\regret(K) \lesssim \sum_{h, k} \pa{ \frac{H}{\sqrt{\gencount_h^k(\tilde{x}_h^k,\tilde{a}_h^k)}}  + \frac{H^2\Xsigmacov}{\gencount_h^k(\tilde{x}_h^k,\tilde{a}_h^k)} } + L_1 KH \sigma.
	\end{align*}
	}
	The term in $1/\gencount_h^k(\tilde{x}_h^k,\tilde{a}_h^k)$ is the \emph{second order term} (in $K$). 

	In the tabular case, it is multiplied by the number of states. Here, it is multiplied by the covering number of the state space $\Xsigmacov$.
	
	From there it remains to bound the sum of the first and second-order terms, and we specifically show that
	{\small
	\begin{align}
	\sum_{h, k} \frac{1}{\sqrt{\gencount_h^k(\tilde{x}_h^k,\tilde{a}_h^k)}} &\lesssim H \sqrt{\sigmacov K} \label{bound:order1}\\
	\text{and } \ \ \ \ \sum_{h, k} \frac{1}{\gencount_h^k(\tilde{x}_h^k,\tilde{a}_h^k)} &\lesssim H \sigmacov\log K,\label{bound:order2}
	\end{align}
	}
	where we note that \eqref{bound:order2} has a worse dependence on $\sigmacov$. As mentioned before, unlike in the tabular case the sum of ``second-order'' terms will actually be the leading term, since the choice of $\sigma$ that minimizes the regret depends on $K$. 
	
	Finally, we obtain that on $\cG\cap \cF$ (of probability $\geq 1-\delta$)
	\begin{align*}
	\regret(K) \lesssim H^2 \sqrt{\sigmacov K} +  H^3 \sigmacov\Xsigmacov + L_1 KH \sigma + H^2\sigmacov\,,
	\end{align*}
	where the extra $H^2\sigmacov$ takes into account the episodes where $\square_h^k > 2\sigma$.

	If the transitions kernels are stationary, \ie $P_1=\ldots=P_H$, the bounds \eqref{bound:order1} and \eqref{bound:order2} can be improved to $\sqrt{\sigmacov KH}$ and $\sigmacov\log (KH)$ respectively, thus improving the final scaling in $H$.\footnote{This is because, in the non-stationary case, we bound the sums over $k$ and then multiply the resulting bound by $H$. In the stationary case, we can directly bound the sums over $(k, h)$.} See App.~\ref{app:variants} for details.

\section{Experiments}
\label{sec:experiments}

To illustrate experimentally the properties of \lipucbvi, we consider a Grid-World environment with continuous states. This Grid-World is composed of two rooms separated by a wall, such that $\statespace = ([0, 1-\Delta] \cup [1+\Delta, 2]) \times [0, 1])$ where $2\Delta=0.1$ is the width of the wall, as illustrated by Figure~\ref{fig:environment}. There are four actions: left, right, up, and down, each one resulting to a displacement of $0.1$ in the corresponding direction. A two-dimensional Gaussian noise is added to the transitions, and, in each room, there is a single region with non-zero reward. The agent has $0.5$ probability of starting in each of the rooms, and the starting position is at the room's bottom left corner. 

We compare \lipucbvi and \lipucbvigreedy to the following baselines:
\begin{enumerate*}[label=(\roman*)]
	\item UCBVI~\cite{Azar2017} using a uniform discretization of the state-space;
	\item OptQL~\cite{Jin2018} also on a uniform discretization;
	\item Adaptive-Q-Learning~\cite{Sinclair2019} that uses an adaptive discretization of the state-space.
\end{enumerate*}
We used the Euclidean distance and the Gaussian kernel with a fixed bandwidth $\sigma=0.025$, matching the granularity of the uniform discretization used by some of the baselines. We also implemented a version of \lipucbvi using the ``expert knowledge'' that the two rooms are equivalent under translation, by using a metric that is invariant with respect to the change of rooms. More details about the experimental setup are provided in Appendix~\ref{sec:app-experiments}.\footnote{Implementations of \lipucbvi are available on  \href{https://github.com/omardrwch/kernel_ucbvi_experiments}{GitHub}, and use the \href{https://github.com/rlberry-py/rlberry}{\texttt{rlberry}} library \cite{rlberry}.}

\begin{figure}[h]
	\centering
	\includegraphics[width=6cm]{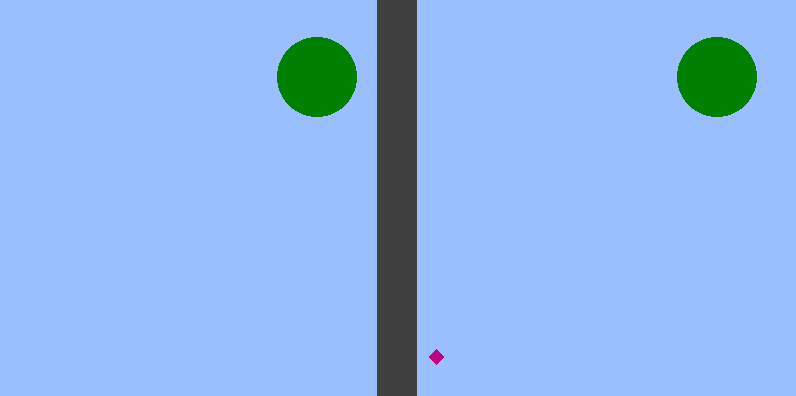}
	\vspace{-0.25cm}
	\caption{Continuous grid-world with two rooms separated by a wall. The circles represent the regions with non-zero rewards.}
	\label{fig:environment}
\end{figure}

We ran the algorithms for $5\times 10^4$ episodes, and Figure~\ref{fig:twinrooms-reward} shows the sum of rewards obtained by each of them. When the curves start behaving as a straight line, it roughly means that the algorithm has converged to a policy whose value is the slope of the line. We see that \lipucbvi is able to outperform the baselines, and that the use of expert knowledge in the kernel design can considerably increase the learning speed. Asymptotically, the extra bias introduced by the kernel (especially its bandwidth) might make \lipucbvi converge to a worse policy at the end: the kernel bandwidth and the discretization width are comparable, but the Gaussian kernel introduces more bias due to sample aggregation. On the other hand, we see that introducing more bias can greatly improve the learning speed in the beginning, especially when expert knowledge is used. This flexibility in handling the bias-variance trade-off is one of the strengths of kernel-based approaches.

\begin{figure}[h]
	\centering
	\includegraphics[width=7cm]{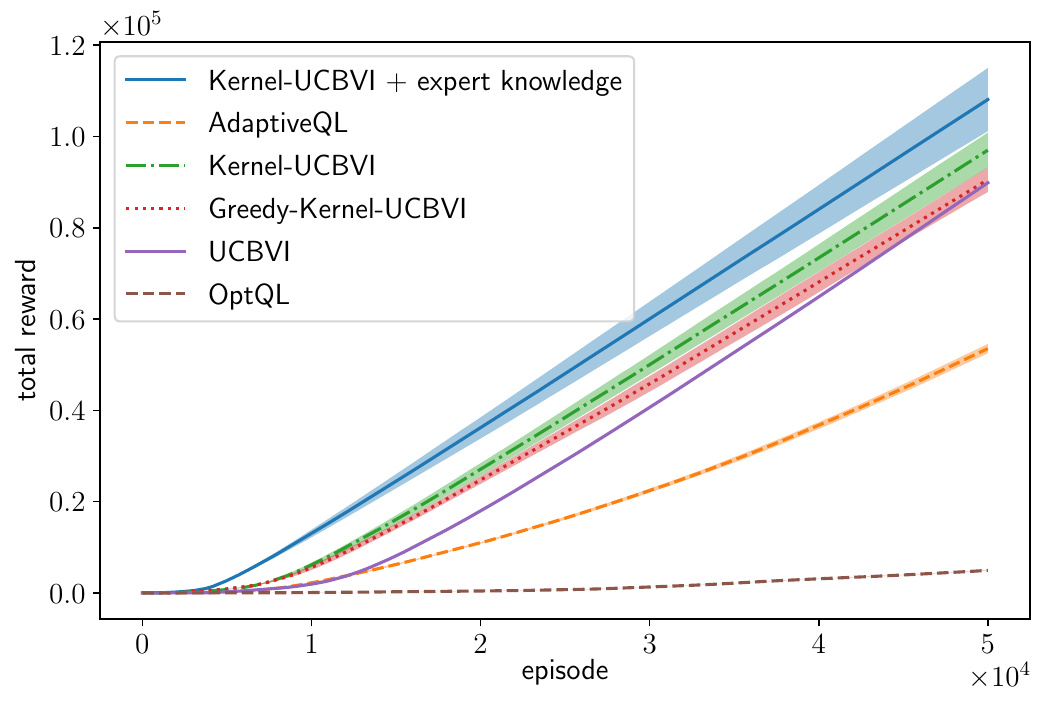}
	\vspace{-0.5cm}
	\caption{Total sum of rewards obtained by \lipucbvi and baselines versus the number of episodes. Average over 8 runs.}
	\label{fig:twinrooms-reward}
\end{figure}

\section{Conclusion}

In this paper, we introduced \lipucbvi, a model-based algorithm for finite-horizon reinforcement learning in metric spaces which employs kernel smoothing to estimate rewards and transitions. By providing new high-probability confidence intervals for weighted sums and non-parametric kernel estimators, we generalize the techniques introduced by \cite{Azar2017} in tabular MDPs to the continuous setting. We prove that the regret of \lipucbvi is of order $H^3 K^{\max\left(\frac{1}{2}, \frac{2d}{2d+1}\right)}$, which is the first regret bound for kernel-based RL using smoothing kernels. In addition, we provide experiments illustrating the effectiveness of \lipucbvi in handling the bias-variance trade-off and in the use of expert knowledge. Interesting directions for future work include the use of learned metrics (\eg using neural networks) and the use of adaptive kernel bandwidths to better handle the bias-variance trade-off asymptotically.

\section*{Acknowledgements}
At Inria and CNRS, this work was supported by the European CHIST-ERA project DELTA, French Ministry of Higher Education and Research, Nord-Pas-de-Calais Regional Council,  French National Research Agency project BOLD (ANR19-CE23-0026-04), FMJH PGMO project 2018-0045. Pierre  M\'enard  is supported by the SFI Sachsen-Anhalt for the project RE-BCI ZS/2019/10/102024 by the Investitionsbank Sachsen-Anhalt.
\bibliography{source_arxiv_mars2022/library2.bib}
\bibliographystyle{icml2021}

\clearpage
\appendix

\onecolumn



\section{Notation and preliminaries}
\label{app:preliminaries}

\subsection{Notation}

Table \ref{tab:notations} presents the main notations used in the proofs. Also, we use the symbol $\lesssim$ with the following meaning:
\begin{align*}
	A \lesssim B \iff A \leq B \times\mathrm{polynomial}\pa{ \log(k),\log(1/\delta), \lambda_r, \lambda_p, \beta, \XAcovdim, \Xcovdim}.
\end{align*}

\begin{table}[h]
	\centering
	\caption{Table of notations}
	\label{tab:notations}
	\begin{tabular}{@{}l|l@{}}
		\toprule
		Notation & Meaning \\ \midrule
		$\distfunc: (\stateactionspace)^2 \to \RR_+$ &  metric on the state-action space $\stateactionspace$ \\
		$\kernel_{\sigma}((x,a), (x', a'))$ & kernel function with bandwidth $\sigma$  \\
		$\kernelfunc: \RR_+ \to [0, 1]$  & ``mother'' kernel function such that $\kernel_{\sigma}(u, v) = \kernelfunc(\dist{u, v}/\sigma)$\\
		$C_1^g, C_2^g$ & positive constants that depend on $\kernelfunc$ (Assumption \ref{assumption:kernel-behaves-as-gaussian}) \\
		$\cN(\epsilon, \stateactionspace, \distfunc)$ & $\epsilon$-covering number of the metric space $(\stateactionspace, \distfunc)$\\
		$\favevent$ & ``good'' event (see Theorem \ref{theorem:favorable-event}) \\
		$\lambda_r,  \lambda_p$ &  Lipschitz constants of rewards and transitions, respectively \\
		$L_h$, for $h \in [H]$ & Lipschitz constant of value functions (see Lemma \ref{lemma:value-funcs-are-liptschitz}) \\
		 $\logplus(x)$  & equal to $\log(x+e)$ \\
		 $\Lip{f}$      & Lipschitz constant of the function $f$ \\
		 $\XAcovdim, \totalcovdim$    & covering dimension of $(\stateactionspace, \distfunc)$ \\
		 $\Xcovdim$     & covering dimension of $(\statespace, \Sdistfunc)$ \\
		 $\sigmacov, \Xsigmacov$ & $\sigma$-covering numbers  of $(\stateactionspace, \distfunc)$ and $(\statespace, \Sdistfunc)$, respectively \\
		\bottomrule
	\end{tabular}
\end{table}

We consider the filtration defined as follows:
\begin{definition}
	Let $\cF_h^k$ be the $\sigma$-algebra generated by the random variables $\braces{x_h^s, a_h^s, x_{h+1}^s, r_h^s}_{s=1}^{k-1} \cup \braces{x_{h'}^k, a_{h'}^k, x_{h'+1}^k, r_{h'}^k}_{h' < h}$, and let $(\cF_h^k)_{k, h}$ be its corresponding filtration.
\end{definition}




\subsection{Preliminaries}

Let $\sigma > 0$. We define the \emph{weights} and the \emph{normalized weights} as
\begin{align*}
	w_h^s(x,a) \eqdef  \kernel_{\sigma}((x,a),(x_h^s, a_h^s)) 
	\quad\text{and}\quad 
	\widetilde{w}_h^s(x,a) \eqdef \frac{ w_h^s(x,a) }{ \beta +  \sum_{l=1}^{k-1}w_h^l(x,a) }
\end{align*}
where $\beta >0$ is a regularization parameter. We define the generalized count at $(x, a)$ at time $(k, h)$ as
$
\gencount_h^k (x, a) \eqdef \beta +  \sum_{s=1}^{k-1}w_h^s(x,a).
$

We define the following estimators for the transition kernels $\braces{P_h}_{h \in [H]}$ and for  for the reward functions $\braces{r_h}_{h \in [H]}$:
\begin{align*}
\widehat{\transition}_h^k(y|x, a) \eqdef  \sum_{s=1}^{k-1} \widetilde{w}_h^s(x,a) \dirac{x_{h+1}^s}(y) 
\quad \text{and}\quad 
\widehat{\reward}_h^k(x, a) \eqdef  \sum_{s=1}^{k-1} \widetilde{w}_h^s(x,a) \reward_h^s.
\end{align*}

For any function $V:\RR \to \RR$, we recall that
\begin{align*}
P_h V(x, a) = \int_{\statespace} V(y)\mathrm{d}P_h(y|x,a) \; \mbox{ and } \; \widehat{P}_h^k V(x, a)  = \sum_{s=1}^{k-1} \widetilde{w}_h^s(x,a) V(x_{h+1}^s).
\end{align*}

We will also using the notion of covering of metric spaces, according to the definition below.

\begin{definition}[covering of a metric space] 
	\label{def:covering}
	Let $\pa{\cU, \rho}$ be a metric space. For any $u \in \cU$, let $\cB(u, \sigma) = \braces{v \in \cU: \rho(u, v)\leq \epsilon}$. We say that a set $\cC_\sigma \subset \cU$ is a $\sigma$-covering of $\pa{\cU, \rho}$ if $\cU \subset \bigcup_{u \in \cC_\sigma}\cB(u, \sigma)$. 
	In addition, we define the $\sigma$-covering number of $\pa{\cU, \rho}$ as  $$\cN(\sigma, \cU, \rho) \eqdef \min \braces{\abs{\cC_\sigma}: \cC_\sigma \text{ is a $\sigma$-covering } of  \pa{\cU, \rho} }.$$
\end{definition}




\section{Description of the algorithm}
\label{app:algorithm}

At the beginning of each episode $k$, the agent has observed the data $\data_h = \braces{(x_h^s, a_h^s, x_{h+1}^s, r_h^s)}_{s \in [k-1]}$ for $h \in [H]$. The number of data tuples in each $\data_h$ is $k-1$.

At each step $h$ of episode $k$, the agent has access to an optimistic value function at step $h+1$, denoted by $V_{h+1}^k$. Using this optimistic value function, the agent computes an upper bound for the $Q$ function at each state-action pair in the data, denoted by $\widetilde{Q}_h^k(x_h^s, a_h^s)$ for $s \in [k-1]$, which we call \emph{optimistic targets}. For any $(x, a)$, we can compute an optimistic target as
\begin{align*}
\widetilde{Q}_h^k(x, a) = \widehat{\reward}_h^k(x, a) + \widehat{P}_h^k V_{h+1}^k(x, a) + \bonus_h^k(x, a)
\end{align*}
where $\bonus_h^k(x, a)$ is an exploration bonus for the pair $(x, a)$ that represents the sum of uncertainties on the transitions and rewards estimates and is defined below:

\begin{fdefinition}[exploration bonus]
	\label{def:exploration_bonus}
	\begin{equation}
	\label{eq:bonus_full}
	\begin{aligned}
	& \bonus_h^k(x, a) = \pbonus_h^k(x, a) +  \rbonus_h^k(x, a) \\
	& = \underbrace{
		\pa{\sqrt{ \frac{H^2 \hoeffdingVarP (k,\delta/6)}{\gencount_h^k(x, a)}} + \frac{\beta H}{\gencount_h^k(x, a)} + \hoeffdingBiasP(k,\delta/6)\sigma }
	}_{\mbox{transition bonus}} 
	 + \underbrace{\pa{\sqrt{\frac{\hoeffdingVarR(k,\delta/6)}{\gencount_h^k(x, a)}} + \frac{\beta}{\gencount_h^k(x, a)} + \hoeffdingBiasR(k,\delta/6) \sigma }}_{\mbox{reward bonus}}
	\end{aligned}
	\end{equation}
	where 
	\begin{align*}
		& \hoeffdingVarR(k,\delta) =  \BigOtilde{\XAcovdim} =  
		2  \log\pa{H \XAcovnumber{\sigma^2/(KH)} \frac{\sqrt{1 + k/\beta}}{\delta}} \\
		& \hoeffdingBiasR(k, \delta)  =  \BigOtilde{L_1 + \sqrt{\XAcovdim}} =  
		\frac{4 C_2^g}{\beta}
		+  \sqrt{\hoeffdingVarR(k,\delta)} \frac{C_2^g}{\beta^{3/2}} 
		+ 
		2 \lambda_r L_1 \pa{1 + \sqrt{\logplus(C_1^g k/\beta)}}  \\
		& \hoeffdingVarP(k,\delta) =  \BigOtilde{\XAcovdim} =  
		2  \log\pa{H \XAcovnumber{\sigma^2/(KH)} \frac{\sqrt{1 + k/\beta}}{\delta}} \\
		& \hoeffdingBiasP(k, \delta)  =  \BigOtilde{L_1 + \sqrt{\XAcovdim}} =  
		\frac{4 C_2^g}{\beta}
		+  \sqrt{\hoeffdingVarP(k,\delta)} \frac{C_2^g}{\beta^{3/2}} 
		+ 
		2 \lambda_p L_1 \pa{1 + \sqrt{\logplus(C_1^g k/\beta)}} 
	\end{align*}
\end{fdefinition}



Then, we build an optimistic $Q$ function $Q_h^k$ by interpolating the optimistic targets:
\begin{equation}
 \label{eq:interpolation}
\forall (x,a), \; Q_h^k(x, a) \eqdef \min_{s\in[k-1]} \left[ \widetilde{Q}_h^k(x_h^{s},a_h^{s}) + L_h\dist{(x,a), (x_h^{s},a_h^{s})} \right]
\end{equation}
and the value function $V_h^k$ is computed as
\begin{align*}
\forall x, \; V_h^k(x) \eqdef \min\pa{H-h+1, \max_{a'} Q_h^k(x, a')}.
\end{align*}

We can check that $(x, a) \mapsto Q_h^k(x, a)$ is $L_h$-Lipschitz with respect to $\distfunc$ and that $(x) \mapsto V_h^k(x)$ is $L_h$-Lipschitz with respect to $\Sdistfunc$.



\section{Concentration}

The first step towards proving our regret bound is to derive confidence intervals for the rewards and transitions, which are presented in propositions \ref{prop:concentration-of-rewards-continuous} and \ref{prop:concentration-of-transitions-continuous}, respectively.

In addition, we need a Bernstein-type inequality for the transition kernels, which is stated in Proposition \ref{prop:uniform-bernstein-transitions}.

Finally, Theorem \ref{theorem:favorable-event} defines a favorable event in which all the confidence intervals that we need to prove our regret bound are valid and we prove that this event happens with high probability.



\subsection{Confidence intervals for the reward functions}

\begin{fproposition}
	\label{prop:concentration-of-rewards-continuous}
	For all $(k, h) \in [K]\times[H]$ and all $(x, a)\in\stateactionspace$, we have
	\begin{align*}
		\abs{\estR_h^k(x, a)-\trueR_h(x, a)} \leq  
		\sqrt{\frac{\hoeffdingVarR(k,\delta)}{\gencount_h^k(x,a)}}
		+ \frac{\beta}{\gencount_h^k(x,a)} + \hoeffdingBiasR(k, \delta)\sigma
	\end{align*}
	with probability at least $1-\delta$, where
	\begin{align*}
		& \hoeffdingVarR(k,\delta) =  \BigOtilde{\XAcovdim} =  
		2  \log\pa{H \XAcovnumber{\sigma^2/(KH)} \frac{\sqrt{1 + k/\beta}}{\delta}} \\
		& \hoeffdingBiasR(k, \delta)  =  \BigOtilde{L_1 + \sqrt{\XAcovdim}} =  
		\frac{4 C_2^g}{\beta}
		+  \sqrt{\hoeffdingVarR(k,\delta)} \frac{C_2^g}{\beta^{3/2}} 
		+ 
		2 \lambda_r L_1 \pa{1 + \sqrt{\logplus(C_1^g k/\beta)}} 
	\end{align*}
\end{fproposition}
\begin{proof}
	The proof is almost identical to the proof of Proposition \ref{prop:concentration-of-transitions-continuous}. The main difference is that the rewards are bounded by $1$, and not by $H$.
\end{proof}




\subsection{Confidence intervals for the transition kernels}

\begin{fproposition}
	\label{prop:concentration-of-transitions-continuous}
	For all $(k, h) \in [K]\times[H]$ and all $(x, a)\in\stateactionspace$, we have
	\begin{align*}
		\abs{\estP_h^k V_{h+1}^*(x, a) -P_h V_{h+1}^*(x, a)} \leq 
		\sqrt{\frac{H^2 \hoeffdingVarP(k,\delta)}{\gencount_h^k(x,a)}}
		+ \frac{\beta H}{\gencount_h^k(x,a)} + \hoeffdingBiasP(k, \delta)\sigma
	\end{align*}
	with probability at least $1-\delta$, where
	\begin{align*}
		& \hoeffdingVarP(k,\delta) =  \BigOtilde{\XAcovdim} =  
		2  \log\pa{H \XAcovnumber{\sigma^2/(KH)} \frac{\sqrt{1 + k/\beta}}{\delta}} \\
		& \hoeffdingBiasP(k, \delta)  =  \BigOtilde{L_1 + \sqrt{\XAcovdim}} =  
		  \frac{4 C_2^g}{\beta}
		  +  \sqrt{\hoeffdingVarP(k,\delta)} \frac{C_2^g}{\beta^{3/2}} 
		+ 
		2 \lambda_p L_1 \pa{1 + \sqrt{\logplus(C_1^g k/\beta)}} 
	\end{align*}
\end{fproposition}
\begin{proof}
	Consider a fixed tuple $(x, a, h)$, and let $V = V_{h+1}^*$. We have:
	\begin{align*}
		\abs{\estP_h^k V(x, a) -P_h V(x, a)} 
		& \leq
		 \abs{
		    \sum_{s=1}^{k-1}\normweight_h^s(x, a)\pa{V(x_{h+1}^s)-\trueP_h V(x,a)}
	    }
    	+\abs{
    	\frac{\beta \trueP_h V(x, a)}{\gencount_h^k(x, a)}
        } \\
    	& \leq
    	\abs{
    		\sum_{s=1}^{k-1}\normweight_h^s(x, a)\pa{V(x_{h+1}^s)-\trueP_h V(x,a)}
    	}
    	+	\frac{\beta H}{\gencount_h^k(x, a)}
	\end{align*}
	since $\norm{V}_\infty \leq H$. Now, by Assumption \ref{assumption:lipschitz-rewards-and-transitions} and the fact that $V$ is $L_1$-Lipschitz:
	\begin{align*}
		& \abs{
			\sum_{s=1}^{k-1}\normweight_h^s(x, a)\pa{V(x_{h+1}^s)-\trueP_h V(x,a)}
		} \\
		& \leq
		 	\abs{	\sum_{s=1}^{k-1}\normweight_h^s(x, a)\pa{V(x_{h+1}^s)-\trueP_h V(x_h^s,a_h^s)}	}
		+   \abs{	\sum_{s=1}^{k-1}\normweight_h^s(x, a)\pa{\trueP_h V(x_h^s,a_h^s)-\trueP_h V(x,a)}	} \\
    	& \leq
			\abs{	\sum_{s=1}^{k-1}\normweight_h^s(x, a)\pa{V(x_{h+1}^s)-\trueP_h V(x_h^s,a_h^s)}	}
		+   L_1\abs{	\sum_{s=1}^{k-1}\normweight_h^s(x, a)\Wassdist{P_h(\cdot|x_h^s,a_h^s), P_h(\cdot|x,a)}	} \\
		& \leq
		\abs{	\sum_{s=1}^{k-1}\normweight_h^s(x, a)\pa{V(x_{h+1}^s)-\trueP_h V(x_h^s,a_h^s)}	}
		+   \lambda_p L_1\abs{	\sum_{s=1}^{k-1}\normweight_h^s(x, a)\dist{(x_h^s, a_h^s),(x, a)}} \\
		& \leq
		\abs{	\sum_{s=1}^{k-1}\normweight_h^s(x, a)\pa{V(x_{h+1}^s)-\trueP_h V(x_h^s,a_h^s)}	}
		+   \lambda_p L_1 2\sigma\pa{1 + \sqrt{\logplus(C_1^g k/\beta)}} 
  	\end{align*} 
  	where, in the last inequality, we used Lemma \ref{lemma:kernel-bias}.
  	
  	Let $W_s \eqdef V(x_{h+1}^s)-\trueP_h V(x_h^s,a_h^s) $. We have $\abs{W_s}\leq 2H$, and $(W_s)_s$ is a martingale difference sequence with respect to the filtration $(\cF_h^s)_s$.  Lemma \ref{lemma:self-normalized-weighted-sum} and an union bound over $h$ gives us:
  	\begin{align*}
  		\abs{	\sum_{s=1}^{k-1}\normweight_h^s(x, a) W_s	} \leq \sqrt{ 2 H^2  \log\pa{\frac{\sqrt{1 + k/\beta}}{\delta}}  \frac{1}{\gencount_h^k(x, a)}}
  	\end{align*}
  	for all $(k, h)$ and fixed $(x, a)$, with probability at least $1-\delta H$. 
  	
	Now, let's extend this inequality for all $(x, a)$ using a covering argument. We define
  	\begin{align*}
  		f_1(x,a) \eqdef \abs{\frac{1}{\gencount_h^k(x, a)}\sum_{s=1}^{k-1} w_h^s(x,a) W_s}
  	\mbox{  and  }
  		f_2(x, a) \eqdef \sqrt{\frac{1}{\gencount_h^k(x, a)}}
  	\end{align*}
\end{proof}
Lemma \ref{lemma:lipschitz-constant-mean-and-bonuses} implies that $\Lip{f_1} \leq 4 C_2^g H k/(\beta\sigma)$ and $\Lip{f_2} \leq \pa{  C_2^g k / \sigma}\beta^{-3/2}$.  Applying Technical Lemma \ref{lemma:covering} using a $\sigma^2/(KH)$-covering of $(\stateactionspace,\distfunc)$, we obtain:
\begin{align*}
	\abs{	\sum_{s=1}^{k-1}\normweight_h^s(x, a) W_s	}  \leq & \sqrt{ 2 H^2  \log\pa{\frac{\sqrt{1 + k/\beta}}{\delta}}  \frac{1}{\gencount_h^k(x, a)}} \\
	& + \frac{\sigma^2}{KH}\Lip{f_1} +  \frac{\sigma^2}{KH} \sqrt{ 2 H^2  \log\pa{\frac{\sqrt{1 + k/\beta}}{\delta}}} \Lip{f_2}
\end{align*}
for all $(x,a,k,h)$ with probability at least $1-\delta H \XAcovnumber{\sigma^2/(KH)}$.

The fact that 
\begin{align*}
	\abs{\estP_h^k V(x, a) -P_h V(x, a)} \leq 
	\abs{	\sum_{s=1}^{k-1}\normweight_h^s(x, a) W_s	}
	+    2 \lambda_p L_1 \sigma\pa{1 + \sqrt{\logplus(C_1^g k/\beta)}} 
	+  \frac{\beta H}{\gencount_h^k(x, a)}
\end{align*}
allows us to conclude.

\subsection{A confidence interval for $P_h f$ uniformly over Lipschitz functions $f$}

In the regret analysis, we will need to control quantities like $(\hat{P}_h^k - P_h)(\hat{f}_h^k)$ for \emph{random} Lipschitz functions $\hat{f}_h^k$, which motivate us to propose a deviation inequality for $(\hat{P}_h^k - P_h)({f})$ which holds uniformly over $f$ in a class of Lipschitz functions. We provide such a result in Proposition~\ref{prop:uniform-bernstein-transitions}.


\begin{fproposition}
	\label{prop:uniform-bernstein-transitions}
	Consider the following function space:
	\begin{align*}
		\cF_{2L_1} \eqdef \braces{f: \statespace\to \RR \mbox{ such that } f \mbox{ is } 2 L_1\mbox{-Lipschitz} \mbox{ and } \norm{f}_\infty \leq 2H }.
	\end{align*}
	With probability at least $1-\delta$, for all $(x, a, h, k) \in \stateactionspace\times[K]\times[H]$ and for all $f\in\cF_{2L_1}$, we have
	\begin{align*}
		\abs{\pa{\estP_h^k-\trueP_h}f(x,a)} \leq & 
		\frac{1}{H}\trueP_h\abs{f}(x, a)
		+ \frac{11H^2\bernvar(k, \delta)+2\beta H}{\gencount_h^k(x, a)} 
		\\
		& 
		+ \bernbias(k, \delta)\sigma^{1+\Xcovdim}
		+ \bernbiastwo(k, \delta)\sigma
	\end{align*}
	with probability at least $1-\delta$, where 
	\begin{align*}
		& \bernvar(k, \delta) = \BigOtilde{\Xsigmacov+\XAcovdim\Xcovdim} 
		 = \log\pa{ \frac{4e(2k+1)}{\delta} H \XAcovnumber{  \frac{\sigma^{2+\Xcovdim}}{KH^2} }\pa{\frac{2H}{L_1\sigma}}^{\Xcovnumber{\sigma}}}
		\\
		& \bernbias(k, \delta)   = \BigOtilde{\Xsigmacov+\XAcovdim\Xcovdim} 
		= \pa{ 
			\frac{2\lambda_p L_1\sigma}{KH^2}
			+ \frac{4C_2^g}{H \beta}
			+ \frac{ 11 C_2^g \bernvar(k, \delta)}{\beta^2}}
		\\
		& \bernbiastwo(k, \delta) = \BigOtilde{L_1} 
		= 32 L_1 +  6 \lambda_p L_1 \pa{1 + \sqrt{\logplus(C_1^g k/\beta)}} 
	\end{align*}
	where $\Xsigmacov = \BigO{1/\sigma^\Xcovdim}$ is the $\sigma$-covering number of $(\statespace,\Sdistfunc)$.
\end{fproposition}
\begin{proof}
	First, consider a fixed tuple $(x,a,h,k)$. Using the same arguments as in the proof of Proposition \ref{prop:concentration-of-transitions-continuous}, we show that:

	\begin{align*}
	\abs{\estP_h^k f(x, a) -P_h f(x, a)} \leq 
	\underbrace{
	\abs{	\sum_{s=1}^{k-1}\normweight_h^s(x, a) W_s(f)	}
	}_{\termA}
	+    4 \lambda_p L_1 \sigma\pa{1 + \sqrt{\logplus(C_1^g k/\beta)}} 
	+  \frac{2 \beta H}{\gencount_h^k(x, a)}
	\end{align*}
	where $W_s(f) \eqdef f(x_{h+1}^s)-\trueP_h f(x_h^s,a_h^s)$. We have $\abs{W_s(f)}\leq 4H$, and $(W_s)_s$ is a martingale difference sequence with respect to the filtration $(\cF_h^s)_s$ for any fixed $f$. We will bound the term $\termA$ using the Bernstein-type inequality given in Lemma \ref{lemma:bernstein-freedman-weighted-sum}. We start by bounding the variance of $f(x_{h+1}^s)$ given $\cF_h^s$:
	\begin{align*}
		\variance{f(x_{h+1}^s)\given \cF_h^s}
		 & = \expect{f(x_{h+1}^s)^2\given \cF_h^s} - \pa{\int_{\statespace}f(y)\rmd\trueP_h(y|x_h^s,a_h^s)}^2 \\
		 & \leq 2 H \expect{\abs{f(x_{h+1}^s)} \given \cF_h^s} \\
		 & = 2 H \trueP_h \abs{f}(x_h^s, a_h^s)
	\end{align*}
	and, consequently,
	\begin{align}
		&  \sum_{s=1}^{k-1}\normweight_h^s(x,a) \variance{f(x_{h+1}^s)\given \cF_h^s} 
		 \leq 2H \sum_{s=1}^{k-1}\normweight_h^s(x,a) \trueP_h \abs{f}(x_h^s, a_h^s) \nonumber\\
		& = 2H \sum_{s=1}^{k-1}\normweight_h^s(x,a) \trueP_h \abs{f}(x, a) + 
		    2H \sum_{s=1}^{k-1}\normweight_h^s(x,a) \pa{ \trueP_h \abs{f}(x_h^s, a_h^s) - \trueP_h \abs{f}(x, a)} \nonumber \\
		& \leq 2H \sum_{s=1}^{k-1}\normweight_h^s(x,a) \trueP_h \abs{f}(x, a) + 4 H \lambda_p L_1 \sum_{s=1}^{k-1}\normweight_h^s(x,a) \dist{(x_h^s, a_h^s)}\nonumber \\
		& \leq 2H \trueP_h \abs{f}(x, a) + 4 H \lambda_p L_1 \sigma\pa{1+\sqrt{\logplus(C_1^g k/\beta)}} \label{eq:aux-bernstein-prop},
	\end{align}
	where, in the last two inequalities, we used Assumption \ref{assumption:lipschitz-rewards-and-transitions} and Lemma \ref{lemma:kernel-bias}.

	Let $\square(k, \delta) = \log(4e(2k+1)/\delta)$.  Using Lemma \ref{lemma:bernstein-freedman-weighted-sum} and the facts that $\sqrt{u+v}\leq\sqrt{u}+\sqrt{v}$ and $\sqrt{uv}\leq(u+v)/2$ for all $u, v > 0$, we obtain 
	\begin{align*}
	\termA & = \abs{\sum_{s=1}^{k-1}\normweight_h^s(x, a) W_s(f)} =  \abs{\frac{\sum_{s=1}^{k-1}\weight_h^s(x, a) W_s(f)}{\beta+\sum_{s=1}^{k-1}\weight_h^s(x, a)}}\\
	& 
	\leq 
	\sqrt{ 2 \square(k, \delta)\frac{ \sum_{s=1}^{k-1}\weight_h^s(x,a)^2 \variance{f(x_{h+1}^s) | \cF_h^s} + 16H^2}{\pa{\beta+\sum_{s=1}^{k-1}\weight_h^s(x, a)}^2}  } 
	+ \frac{(8/3) H \square(k, \delta)}{\beta+\sum_{s=1}^{k-1}\weight_h^s(x, a)} \nonumber \\
	& \leq \frac{H^2\square(k, \delta)}{\beta+\sum_{s=1}^{k-1}\weight_h^s(x, a)}
	+ \frac{1}{2H^2} \frac{ \sum_{s=1}^{k-1}\weight_h^s(x,a)^2 \variance{f(x_{h+1}^s) | \cF_h^s} }{\beta+\sum_{s=1}^{k-1}\weight_h^s(x, a)}
	+  \frac{(8/3 + 4\sqrt{2}) H \square(k, \delta)}{\beta+\sum_{s=1}^{k-1}\weight_h^s(x, a)} \\
	& \leq \frac{(H^2+10H)\square(k, \delta)}{\gencount_h^k(x,a)}
	+ \frac{1}{2H^2} \frac{ \sum_{s=1}^{k-1}\weight_h^s(x,a) \variance{f(x_{h+1}^s) | \cF_h^s} }{\beta+\sum_{s=1}^{k-1}\weight_h^s(x, a)} \\
	& = \frac{(H^2+10H)\square(k, \delta)}{\gencount_h^k(x,a)}
	+ \frac{1}{2H^2}  \sum_{s=1}^{k-1}\normweight_h^s(x,a) \variance{f(x_{h+1}^s) | \cF_h^s}
	\end{align*}	
	for all $k$, with probability at least $1-\delta$.  Above, we also used the fact that $\weight_h^s(x, a)^2 \leq \weight_h^s(x, a)$ since the weights are in $[0,1]$.
	
	Inequality \ref{eq:aux-bernstein-prop} yields
	\begin{align*}
		\termA &
		\leq \frac{1}{H} \trueP_h \abs{f}(x, a) + \frac{(H^2 + 10 H) \square(k, \delta)}{\gencount_h^k(x, a)} +  \frac{2 \lambda_p L_1 \sigma}{H}\pa{1+\sqrt{\logplus(C_1^g k/\beta)}}
	\end{align*}
	with probability at least $1-\delta$.
	
	\paragraph{Extending to all $(x, a)$} Assumption  \ref{assumption:lipschitz-rewards-and-transitions} implies that $(x, a)\mapsto (1/H)\trueP_h \abs{f}(x, a)$ is $2\lambda_p L_1$-Lipschitz. Let 
	\begin{align*}
		f_1(x, a) = \abs{\sum_{s=1}^{k-1} \normweight_h^s(x, a)W_s(f)}
		\quad\text{and}\quad 
		f_2(x, a) = \frac{1}{\gencount_h^k(x, a)}.
	\end{align*}
	Lemma \ref{lemma:lipschitz-constant-mean-and-bonuses} implies that $\Lip{f_1}\leq 4HC_2^gk/(\beta\sigma)$ and $\Lip{f_2}\leq C_2^gk/(\beta^2\sigma)$.  Applying Technical Lemma \ref{lemma:covering} using a $\sigma^{2+\Xcovdim}/(KH^2)$-covering of $(\stateactionspace,\distfunc)$, and doing an union bound over $[H]$, we obtain:
	\begin{align*}
		 \abs{\sum_{s=1}^{k-1}\normweight_h^s(x, a) W_s(f)} 
		 \leq & 
		 \frac{1}{H} \trueP_h \abs{f}(x, a) + \frac{(H^2 + 10 H) \square(k, \delta)}{\gencount_h^k(x, a)} +  \frac{2 \lambda_p L_1 \sigma}{H}\pa{1+\sqrt{\logplus(C_1^g k/\beta)}} \\
		 & + \frac{\sigma^{2+\Xcovdim}}{KH^2}\pa{ 
		 	2\lambda_p L_1
		 	+ \frac{4HC_2^gk}{\beta\sigma}
	 		+ \frac{C_2^gk (H^2 + 10 H) \square(k, \delta)}{\beta^2\sigma}
 		   }
	\end{align*}
	for all $(x,a,h,k)$ with probability at least $1-\delta H \XAcovnumber{  \frac{\sigma^{2+\Xcovdim}}{KH^2} }$.
	
	\paragraph{Extending to all $f \in \cF_{2L_1}$} The inequalities above give us

	\begin{align*}
		\abs{\estP_h^k f(x, a) -P_h f(x, a)} \leq 
		& 
		\frac{1}{H} \trueP_h \abs{f}(x, a) + \frac{(H^2 + 10 H) \square(k, \delta)}{\gencount_h^k(x, a)}  \\
		& + \frac{\sigma^{2+\Xcovdim}}{KH^2}\pa{ 
			2\lambda_p L_1
			+ \frac{4HC_2^gk}{\beta\sigma}
			+ \frac{C_2^gk (H^2 + 10 H) \square(k, \delta)}{\beta^2\sigma}
		   } \\
	    & +    6 \lambda_p L_1 \sigma\pa{1 + \sqrt{\logplus(C_1^g k/\beta)}} 
	    +  \frac{2 \beta H}{\gencount_h^k(x, a)}
	\end{align*}
	for all $(x,a,h,k)$ with probability at least $1-\delta H \XAcovnumber{  \frac{\sigma^{2+\Xcovdim}}{KH^2} }$. 
	
	According to Lemma \ref{lemma:lipschitz-covering}, the $8L_1\sigma$-covering number of $\cF_{2L_1}$ is bounded by $(2H/(L_1\sigma))^{\Xcovnumber{\sigma}}$. The functions $f\mapsto \abs{\estP_h^k f(x, a) -P_h f(x, a)}$ and $f\mapsto \frac{1}{H} \trueP_h \abs{f}(x, a)$ are 2-Lipschitz with respect to $\norm{\cdot}_\infty$. Hence, Lemma \ref{lemma:covering} gives us:
	\begin{align*}
		\abs{\estP_h^k f(x, a) -P_h f(x, a)} \leq 
		& 
		\frac{1}{H} \trueP_h \abs{f}(x, a) + \frac{(H^2 + 10 H) \square(k, \delta)}{\gencount_h^k(x, a)}  \\
		& + \frac{\sigma^{2+\Xcovdim}}{KH^2}\pa{ 
			2\lambda_p L_1
			+ \frac{4HC_2^gk}{\beta\sigma}
			+ \frac{C_2^gk (H^2 + 10 H) \square(k, \delta)}{\beta^2\sigma}
		} \\
		& +    6 \lambda_p L_1 \sigma\pa{1 + \sqrt{\logplus(C_1^g k/\beta)}} 
		+  \frac{2 \beta H}{\gencount_h^k(x, a)} \\
		& + 32 L_1\sigma
	\end{align*}
	for all $(x,a,h,k)$ with probability at least $1-\delta H \XAcovnumber{  \frac{\sigma^{2+\Xcovdim}}{KH^2} }(2H/(L_1\sigma))^{\Xcovnumber{\sigma}}$, which concludes the proof.
	
\end{proof}



\subsection{Good event}
\begin{ftheorem}[Good event]
	\label{theorem:favorable-event}
	Let $\favevent = \favevent_1 \cup \favevent_2 \cup \favevent_3$, where
	\begin{align*}
		 & \favevent_1 \eqdef \braces{ \forall(x,a,k,h),\; 	\abs{\estR_h^k(x, a)-\trueR_h(x, a)} \leq  
			\sqrt{\frac{\hoeffdingVarR(k,\delta/6)}{\gencount_h^k(x,a)}}
			+ \frac{\beta}{\gencount_h^k(x,a)} + \hoeffdingBiasR(k, \delta/6)\sigma}
		\\
		 & \favevent_2 \eqdef \braces{ \forall(x,a,k,h),\;
		 \abs{\estP_h^k V_{h+1}^*(x, a) -P_h V_{h+1}^*(x, a)} \leq 
		 \sqrt{\frac{H^2 \hoeffdingVarP(k,\delta/6)}{\gencount_h^k(x,a)}}
		 + \frac{\beta H}{\gencount_h^k(x,a)} + \hoeffdingBiasP(k, \delta/6)\sigma
		}
		\\
		& \favevent_3 \eqdef \Bigg\lbrace 
		 \forall(x,a,k,h,f),\;
		 \abs{\pa{\estP_h^k-\trueP_h}f(x,a)} \leq 
		 \frac{1}{H}\trueP_h\abs{f}(x, a)
		 + \frac{11H^2\bernvar(k, \delta/6)+2\beta H}{\gencount_h^k(x, a)} 
		 \\
		 & \text{~~~~~~~~~~~~~~~~~~~~~~~~~~~~~~~~~~~~~~~~~~~~~~~~~~~~~~~~~~~~~~~~~~~~~~}
		 + \bernbias(k, \delta/6)\sigma^{1+\Xcovdim}
		 + \bernbiastwo(k, \delta/6)\sigma
		\Bigg\rbrace
	\end{align*}
	for $(x,a,k,h)\in \stateactionspace\times[K]\times[H]$ and $f\in\cF_{2L_1}$, and where
	\begin{align*}
	& \hoeffdingVarR(k,\delta) =  \BigOtilde{\XAcovdim}
	,\quad 
	\hoeffdingBiasR(k, \delta)  =  \BigOtilde{L_1 + \sqrt{\XAcovdim}}
	\\
	& \hoeffdingVarP(k,\delta) =  \BigOtilde{\XAcovdim}
	,\quad 
	\hoeffdingBiasP(k, \delta)  =  \BigOtilde{L_1 + \sqrt{\XAcovdim}},
	\\
	& \bernvar(k, \delta) = \BigOtilde{\Xsigmacov+\XAcovdim\Xcovdim}
	,\quad
	\bernbias(k, \delta)   = \BigOtilde{\Xsigmacov+\XAcovdim\Xcovdim}
	,\quad
	 \bernbiastwo(k, \delta) = \BigOtilde{L_1} 
	\end{align*}
	are defined in Propositions \ref{prop:concentration-of-rewards-continuous}, \ref{prop:concentration-of-transitions-continuous} and \ref{prop:uniform-bernstein-transitions}. Then,
	\begin{align*}
		\prob{\favevent} \geq 1 - \delta/2.
	\end{align*}
\end{ftheorem}
\begin{proof}
	Immediate consequence of Propositions \ref{prop:concentration-of-rewards-continuous}, \ref{prop:concentration-of-transitions-continuous} and \ref{prop:uniform-bernstein-transitions}.
\end{proof}



\section{Optimism and regret bound }
\label{app:optimism-regret}

\begin{fproposition}[Optimism]
	\label{prop:optimism}
	In the event $\favevent$, whose probability is greater than $1-\delta/2$, we have:
	\begin{align*}
	\forall (x,a), \; Q_h^k(x,a) \geq Q_h^*(x,a)
	\end{align*}
\end{fproposition}
\begin{proof}
	We proceed by induction.
	
	\emph{Initialization \ \ } When $h = H+1$, we have $Q_h^k(x,a) = Q_h^*(x,a) = 0$ for all $(x, a)$.
	
	\emph{Induction hypothesis \ \ } Assume that  $Q_{h+1}^k(x,a) \geq Q_{h+1}^*(x,a)$ for all $(x, a)$.
	
	\emph{Induction step \ \ } The induction hypothesis implies that $V_{h+1}^k(x) \geq V_{h+1}^*(x)$ for all $x$. Hence, for all $(x, a)$, we have
	\begin{small}
		\begin{align*}
		\widetilde{Q}_h^k(x,a) - Q_h^*(x,a) = \underbrace{ (\estR_h^k(x, a) - \reward_h(x, a))+ (\widehat{P}_h^k-P_h)V_{h+1}^*(x, a) + \bonus_h^k(x,a)}_{ \geq 0 \mbox{ in } \favevent} + \underbrace{\estP_h^k (V_{h+1}^k - V_{h+1}^*)(x,a)}_{\geq 0 \mbox{ by induction hypothesis}} \geq 0.
		\end{align*}
	\end{small}
	In particular $\widetilde{Q}_h^k(x_h^s,a_h^s) - Q_h^*(x_h^s,a_h^s) \geq 0$ for all $s \in [k-1]$. This implies that
	\begin{align*}
	\widetilde{Q}_h^k(x_h^{s},a_h^{s}) + L_h\dist{(x,a), (x_h^{s},a_h^{s})} \geq Q_h^*(x_h^{s},a_h^{s}) + L_h\dist{(x,a), (x_h^{s},a_h^{s})} \geq Q_h^*(x,a)
	\end{align*}
	for all $s \in [k-1]$, since $Q_h^*$ is $L_h$-Lipschitz. Finally, we obtain
	\begin{align*}
	\forall (x,a), \; Q_h^k(x, a) = \min_{s\in[k-1]} \left[ \widetilde{Q}_h^k(x_h^{s},a_h^{s}) + L_h\dist{(x,a), (x_h^{s},a_h^{s})} \right] \geq Q_h^*(x,a).
	\end{align*}
\end{proof}

\begin{fcorollary}
	\label{corollary:upper-bound-regret-with-deltas}
	Let $ \delta_h^k \eqdef V_h^k(x_h^k) - V_h^{\pi_k}(x_h^k)$. Then, on $\favevent$,
	$\regret(K) \leq \sum_{k=1}^K \delta_1^k$.
\end{fcorollary}
\begin{proof} Combining the definition of the regret with Proposition~\ref{prop:optimism} easily yields, on the event $\favevent$,
	\begin{eqnarray*}
		\regret(K) &=& \sum_{k=1}^K \left(V_1^\star(x_1^k) - V_1^{\pi_k}(x_1^k) \right)  = \sum_{k=1}^K \left(\max_a Q_1^\star(x_1^k,a) - V_1^{\pi_k}(x_1^k) \right)  \\
		&\leq & \sum_{k=1}^K \left(\min\left[H-h +1, \max_a Q_1^k(x_1^k,a)\right] - V_1^{\pi_k}(x_1^k) \right) = \sum_{k=1}^K \left(V_1^k(x_1^k,a) - V_1^{\pi_k}(x_1^k) \right)\;,
	\end{eqnarray*}
\end{proof}

\begin{definition}
	For any $(k,h)$, we define $(\tx_h^k, \ta_h^k)$ as state-action pair in the past data $\data_h$ that is the closest to  $(x_h^k, a_h^k)$, that is
	\begin{align*}
	& (\tx_h^k, \ta_h^k) \eqdef \argmin_{(x_h^s, a_h^s): s < k} \dist{(x_h^k, a_h^k), (x_h^s, a_h^s)}\,.
	\end{align*}
\end{definition}

\begin{fproposition}
	\label{prop:regret-aux-result}
	With probability $1-\delta$, the regret of $\lipucbvi$ is bounded as follows
	\begin{align*}
	\regret(K) \lesssim  & 
	H^2 \sigmacov + L_1  KH \sigma + \sum_{k=1}^K\sum_{h=1}^H \pa{1+\frac{1}{H}}^{h} \widetilde{\xi}_{h+1}^k \\
	&  + \sum_{k=1}^K\sum_{h=1}^H
	\pa{  
		\frac{H}{\sqrt{\gencount_h^k(\tx_h^k, \ta_h^k)}} + 
		\frac{H^2 \Xsigmacov}{\gencount_h^k(\tx_h^k, \ta_h^k)}
	}
	\indic{\dist{(\tx_h^k, \ta_h^k), (x_h^k, a_h^k)} \leq 2 \sigma} 
	\end{align*}
	where $\widetilde{\xi}_{h+1}^k$ is a martingale difference sequence with respect to $(\cF_h^k)_{k, h}$ such that $\abs{\widetilde{\xi}_{h+1}^k} \leq 4H$.
\end{fproposition}
\begin{proof}
	On $\favevent$, we have
	\begin{align*}
	\delta_h^k & = V_h^k(x_h^k) - V_h^{\pi_k}(x_h^k) \\
	& \leq Q_h^k(x_h^k,a_h^k) - Q_h^{\pi_k}(x_h^k,a_h^k) \\
	& \leq Q_h^k(\tx_h^k,\ta_h^k) - Q_h^{\pi_k}(x_h^k, a_h^k) 
	+  L_1 \dist{(\tx_h^k, \ta_h^k), (x_h^k, a_h^k)}, \mbox{ since $Q_h^k$ is $L_1$-Lipschitz } \\
	& \leq \widetilde{Q}_h^k(\tx_h^k,\ta_h^k) - Q_h^{\pi_k}(x_h^k,a_h^k) 
	+   L_1  \dist{(\tx_h^k, \ta_h^k), (x_h^k, a_h^k)},
	\mbox{ since $Q_h^k(\tx_h^k,\ta_h^k) \leq \widetilde{Q}_h^k(\tx_h^k,\ta_h^k)$ by definition of $Q_h^k$ } \\
	& = \estR_h^k(\tx_h^k, \ta_h^k) - \reward_h(x_h^k, a_h^k) 
	+  L_1  \dist{(\tx_h^k, \ta_h^k), (x_h^k, a_h^k)} 
	+ \bonus_h^k(\tx_h^k,\ta_h^k) 
	+  \estP_h^k V_{h+1}^k(\tx_h^k, \ta_h^k)-P_h V_{h+1}^{\pi_k}(x_h^k, a_h^k)  \\
	&  = \underbrace{ 
	     \estR_h^k(\tx_h^k, \ta_h^k) - \reward_h(x_h^k, a_h^k) 
	    }_{\termA}
	 + 
	     \underbrace{ 
	     	\sqrbrackets{\estP_h^k - P_h} V_{h+1}^*(\tx_h^k,\ta_h^k)
	      }_{\termB}
	 + \underbrace{ 
		  \sqrbrackets{\estP_h^k - P_h}\pa{V_{h+1}^k-V_{h+1}^*}(\tx_h^k, \ta_h^k)
		}_{\termC}
	 \\ 
	& + \underbrace{ \trueP_h V_{h+1}^k(\tx_h^k, \ta_h^k) - P_h V_{h+1}^{\pi_k}(x_h^k, a_h^k) 
    	}_{\termD}	
	 +  L_1  \dist{(\tx_h^k, \ta_h^k), (x_h^k, a_h^k)}	+ \bonus_h^k(\tx_h^k,\ta_h^k) 
	\end{align*}
	
	Now, let's bound each of the terms $\termA, \termB, \termC$ and $\termD$

	\paragraph{Term $\termA$} Using the fact that $\trueR_h$ is $\lambda_r$-Lipschitz and the definition of $\favevent$:
	\begin{align*}
		\termA & = \estR_h^k(\tx_h^k, \ta_h^k) - \reward_h(x_h^k, a_h^k)  \leq \lambda_r\dist{(\tx_h^k, \ta_h^k), (x_h^k, a_h^k)} + \rbonus_h^k(\tx_h^k, \ta_h^k) \\
		& \lesssim 
		 \lambda_r\dist{(\tx_h^k, \ta_h^k), (x_h^k, a_h^k)}
		 +   
		 \sqrt{\frac{1}{\gencount_h^k(x,a)}}
		 + \frac{\beta}{\gencount_h^k(x,a)} + L_1\sigma
		 .
	\end{align*}
	
	\paragraph{Term $\termB$} Using the definition of $\favevent$:
	\begin{align*}
		\termB = \sqrbrackets{\estP_h^k - P_h} V_{h+1}^*(\tx_h^k,\ta_h^k) \lesssim 
		\sqrt{\frac{H^2}{\gencount_h^k(x,a)}}
		+ \frac{\beta H}{\gencount_h^k(x,a)} + L_1\sigma
	\end{align*}
	
	\paragraph{Term $\termC$} Using again the definition of $\favevent$, where $V_{h+1}^k\geq V_{h+1}^*$, and the fact that $V_{h+1}^*\geq V_{h+1}^{\pi_k}$:
	\begin{align*}
		\termC 
		& = \sqrbrackets{\estP_h^k - P_h}\pa{V_{h+1}^k-V_{h+1}^*}(\tx_h^k, \ta_h^k) \\
		& \lesssim \frac{1}{H}\trueP_h\pa{V_{h+1}^k-V_{h+1}^*}(\tx_h^k, \ta_h^k)
		+ \frac{H^2\Xsigmacov}{\gencount_h^k(\tx_h^k, \ta_h^k)} 
		+ L_1\sigma \\
		& \leq \frac{1}{H}\trueP_h\pa{V_{h+1}^k-V_{h+1}^*}(x_h^k, a_h^k) + 2 \lambda_p L_1\dist{(\tx_h^k, \ta_h^k), (x_h^k, a_h^k)}
		+ \frac{H^2\Xsigmacov}{\gencount_h^k(\tx_h^k, \ta_h^k)} 
		+ L_1\sigma \\
		& \lesssim \frac{1}{H}\trueP_h\pa{V_{h+1}^k-V_{h+1}^{\pi_k}}(x_h^k, a_h^k) +  L_1\dist{(\tx_h^k, \ta_h^k), (x_h^k, a_h^k)}
		+ \frac{H^2\Xsigmacov}{\gencount_h^k(\tx_h^k, \ta_h^k)} 
		+ L_1\sigma \\
		& = \frac{1}{H}\pa{ \delta_{h+1}^k + \xi_{h+1}^k}  +  L_1\dist{(\tx_h^k, \ta_h^k), (x_h^k, a_h^k)}
		+ \frac{H^2\Xsigmacov}{\gencount_h^k(\tx_h^k, \ta_h^k)} 
		+ L_1\sigma
	\end{align*}
	where 
	\begin{align*}
		\xi_{h+1}^k = \trueP_h\pa{V_{h+1}^k-V_{h+1}^{\pi_k}}(x_h^k, a_h^k) - \delta_{h+1}^k
	\end{align*}
	is a martingale difference sequence with respect to $(\cF_h^k)_{k,h}$ bounded by $4H$.

	\paragraph{Term $\termD$} We have 
	\begin{align*}
		\termD & = \trueP_h V_{h+1}^k(\tx_h^k, \ta_h^k) - P_h V_{h+1}^{\pi_k}(x_h^k, a_h^k) \\
		& \leq \lambda_p L_1 \dist{(\tx_h^k, \ta_h^k), (x_h^k, a_h^k)} + \trueP_h V_{h+1}^k(x_h^k, a_h^k) - P_h V_{h+1}^{\pi_k}(x_h^k, a_h^k) \\
		& = \delta_{h+1}^k + \xi_{h+1}^k + \lambda_p L_1 \dist{(\tx_h^k, \ta_h^k), (x_h^k, a_h^k)}.
	\end{align*}
	
	Putting together the bounds above, we obtain
	\begin{align*}
		\delta_h^k \lesssim
		 \pa{1+\frac{1}{H}}\pa{\delta_{h+1}^k + \xi_{h+1}^k}
		 +  L_1 \dist{(\tx_h^k, \ta_h^k), (x_h^k, a_h^k)}		
		 +  \sqrt{\frac{H^2}{\gencount_h^k(\tx_h^k, \ta_h^k)}}
		 +  \frac{H^2\Xsigmacov}{\gencount_h^k(\tx_h^k, \ta_h^k)} 
		 + L_1\sigma
	\end{align*}
	where the constant in front of $\delta_{h+1}^k$ is exact (not hidden by $\lesssim$). 
	
	Now, consider the event $E_h^k \eqdef \braces{ \dist{(\tx_h^k, \ta_h^k), (x_h^k, a_h^k)} \leq 2\sigma}$ and let $\overline{E}_h^k$ be its complement.  Using the fact that $\delta_{h+1}^k \geq 0$ on $\favevent$, the inequality above implies
	\begin{align}
		& \indic{E_h^k}  \delta_h^k  \lesssim
		\indic{E_h^k} \pa{1+\frac{1}{H}}\pa{\delta_{h+1}^k + \xi_{h+1}^k}
		+  3 L_1 \sigma 	
		+  \indic{E_h^k}\sqrt{\frac{H^2}{\gencount_h^k(\tx_h^k, \ta_h^k)}} 
		+  \indic{E_h^k} \frac{H^2\Xsigmacov}{\gencount_h^k(\tx_h^k, \ta_h^k)} 
		\nonumber
		\\
		& 
		\lesssim
		 \pa{1+\frac{1}{H}}\pa{\delta_{h+1}^k + \indic{E_h^k} \xi_{h+1}^k}
		+  3 L_1 \sigma 	
		+  \indic{E_h^k}\sqrt{\frac{H^2}{\gencount_h^k(\tx_h^k, \ta_h^k)}} 
		+  \indic{E_h^k} \frac{H^2\Xsigmacov}{\gencount_h^k(\tx_h^k, \ta_h^k)} \label{eq:trick-delta-one}.
	\end{align}
	Now, using the fact that $\delta_h^k \leq H$, we obtain
	\begin{align}
		& \delta_h^k 
		= \indic{E_h^k} \delta_h^k + \indic{\overline{E}_h^k} \delta_h^k 
		\label{eq:trick-delta-two}
		\\
		& \leq \indic{E_h^k} \delta_h^k + H \indic{\overline{E}_h^k}
		\nonumber
		 \\
		& \lesssim H \indic{\overline{E}_h^k} + \pa{1+\frac{1}{H}}\pa{\delta_{h+1}^k + \indic{E_h^k} \xi_{h+1}^k}
		+  3 L_1 \sigma 	
		+  \indic{E_h^k}\sqrt{\frac{H^2}{\gencount_h^k(\tx_h^k, \ta_h^k)}} 
		+  \indic{E_h^k} \frac{H^2\Xsigmacov}{\gencount_h^k(\tx_h^k, \ta_h^k)}.
		\nonumber
	\end{align}
	
	This yields
	\begin{align*}
		\delta_1^k   \lesssim & \sum_{h=1}^H  \indic{E_h^k} \pa{\sqrt{\frac{H^2}{\gencount_h^k(\tx_h^k, \ta_h^k)}}
		   +  \frac{H^2\Xsigmacov}{\gencount_h^k(\tx_h^k, \ta_h^k)} }
		   \\
		   &
		   + \sum_{h=1}^H \pa{1+\frac{1}{H}}^h   \indic{E_h^k} \xi_{h+1}^k 
		    + L_1 H \sigma
		   + H \sum_{h=1}^H \indic{\overline{E}_h^k}.
	\end{align*}
	
	Let $\widetilde{\xi}_{h+1}^k \eqdef \indic{E_h^k} \xi_{h+1}^k$. We can verify that $\widetilde{\xi}_{h+1}^k$ is a martingale difference sequence with respect to $(\cF_h^k)_{k,h}$ bounded by $4H$.
	
	Applying Corollary \ref{corollary:upper-bound-regret-with-deltas}, we obtain:
	\begin{align*}
		\regret(K) 
		\leq \sum_{k=1}^K \delta_1^k
		&  \lesssim
		 \sum_{k=1}^K\sum_{h=1}^H  \indic{E_h^k} \pa{\sqrt{\frac{H^2}{\gencount_h^k(\tx_h^k, \ta_h^k)}}
		 	+  \frac{H^2\Xsigmacov}{\gencount_h^k(\tx_h^k, \ta_h^k)} }
		 \\
		 &
		 +  \sum_{k=1}^K\sum_{h=1}^H \pa{1+\frac{1}{H}}^h   \widetilde{\xi}_{h+1}^k 
		 + L_1 K H \sigma
		 + H  \sum_{k=1}^K \sum_{h=1}^H \indic{\overline{E}_h^k}.
	\end{align*}

	Finally, we bound the sum
	\begin{align*}
		H \sum_{k=1}^K \sum_{h=1}^H \indic{\overline{E}_h^k}
		= H \sum_{h=1}^H \sum_{k=1}^K  \indic{ \dist{(\tx_h^k, \ta_h^k), (x_h^k, a_h^k)} > 2\sigma} \leq H^2 \sigmacov
	\end{align*}
	since, for each $h$, the number of episodes where the event $\braces{ \dist{(\tx_h^k, \ta_h^k), (x_h^k, a_h^k)} > 2\sigma}$ occurs is bounded by $\sigmacov$. Recalling the definition $E_h^k \eqdef \braces{ \dist{(\tx_h^k, \ta_h^k), (x_h^k, a_h^k)} \leq 2\sigma}$, this concludes the proof.

\end{proof}

\begin{fproposition}
	\label{prop:bounding-sum-of-bonuses}
	We have
	\begin{align*}
	\sum_{k=1}^K\sum_{h=1}^H \frac{1}{\gencount_h^k(\tx_h^k,\ta_h^k)}\indic{\dist{(\tx_h^k, \ta_h^k), (x_h^k, a_h^k)} \leq 2\sigma} 
	\lesssim H \sigmacov
	\end{align*}
	and
	\begin{align*}
	\sum_{k=1}^K\sum_{h=1}^H \frac{1}{\sqrt{\gencount_h^k(\tx_h^k,\ta_h^k)}}\indic{\dist{(\tx_h^k, \ta_h^k), (x_h^k, a_h^k)} \leq 2 \sigma}
	\lesssim  H \sigmacov  + H \sqrt{\sigmacov K}.
	\end{align*}
\end{fproposition}
\begin{proof}
	First, we will need some definitions. Let $\cC_\sigma = \braces{(x_j, a_j)\in\stateactionspace, j = 1,\ldots, \sigmacov}$ be a $\sigma$-covering of $(\stateactionspace, \distfunc)$. We define a partition $\braces{B_j}_{j=1}^{\sigmacov}$ of $\stateactionspace$ as follows:
	\begin{align*}
		B_j = \braces{ (x, a) \in \stateactionspace: (x_j, a_j) = \argmin_{(x_i, a_i)\in\cC_\sigma} \dist{(x,a), (x_i, a_i)} }
	\end{align*}
	where ties in the $\argmin$ are broken arbitrarily.
	
	We define the number of visits to each set $B_j$ as 
	$\partitioncount_h^k(B_j) \eqdef \sum_{s=1}^{k-1}\indic{(x_h^s, a_h^s) \in B_j}.$
	
	Now, assume that $(x_h^k, a_h^k)\in B_j$. If, in addition, $\dist{(\tx_h^k, \ta_h^k), (x_h^k, a_h^k)} \leq 2 \sigma$, we obtain
	\begin{align*}
		\gencount_h^k(\tx_h^k,\ta_h^k) & = \beta +  \sum_{s=1}^{k-1}\kernel_{\sigma}((\tx_h^k,\ta_h^k),(x_h^s, a_h^s)) 
		 = \beta + \sum_{s=1}^{k-1} \kernelfunc\pa{\frac{\dist{(\tx_h^k,\ta_h^k),(x_h^s, a_h^s)}}{\sigma}} \\
		& \geq \beta + \sum_{s=1}^{k-1} \kernelfunc\pa{\frac{\dist{(\tx_h^k,\ta_h^k),(x_h^s, a_h^s)}}{\sigma}} \indic{(x_h^s, a_h^s) \in B_j} \\
		& \geq \beta + \kernelfunc(4) \sum_{s=1}^{k-1}\indic{(x_h^s, a_h^s) \in B_j}
		=  \beta\pa{ 1 + \kernelfunc(4)\beta^{-1}\partitioncount_h^k(B_j)}
	\end{align*}
	since, if $(x_h^s, a_h^s) \in B_j$, we have $\dist{(\tx_h^k,\ta_h^k),(x_h^s, a_h^s)} \leq 4 \sigma$ and we use the fact that $\kernelfunc$ is non-increasing. 
	
	We are now ready to bound the sums involving $1/\gencount_h^k(\tx_h^k,\ta_h^k)$. We will use the fact that $\kernelfunc(4) > 0$ by Assumption \ref{assumption:kernel-behaves-as-gaussian}.

	\paragraph{Bounding the sum of the first order terms}
	\begin{align*}
	& \sum_{k=1}^K\sum_{h=1}^H \sqrt{\frac{1}{\gencount_h^k(\tx_h^k,\ta_h^k)}}\indic{\dist{(\tx_h^k, \ta_h^k), (x_h^k, a_h^k)} \leq 2\sigma} \\
	& = \sum_{k=1}^K\sum_{h=1}^H \sum_{j=1}^{\abs{\cC_\sigma}} \sqrt{\frac{1}{\gencount_h^k(\tx_h^k,\ta_h^k)}} \indic{\dist{(\tx_h^k, \ta_h^k), (x_h^k, a_h^k)} \leq 2\sigma} \indic{(x_h^k, a_h^k) \in B_j} \\
	& \leq \beta^{-1/2} \sum_{k=1}^K\sum_{h=1}^H \sum_{j=1}^{\abs{\cC_\sigma}} \frac{1}{\sqrt{1 + \kernelfunc(4)\beta^{-1}\partitioncount_h^k(B_j)}} \indic{\dist{(\tx_h^k, \ta_h^k), (x_h^k, a_h^k)} \leq 2\sigma} \indic{(x_h^k, a_h^k) \in B_j} \\
	& \leq \beta^{-1/2}\sum_{h=1}^H \sum_{j=1}^{\abs{\cC_\sigma}}\sum_{k=1}^K \frac{ \indic{(x_h^k, a_h^k) \in B_j}}{\sqrt{1 + \kernelfunc(4)\beta^{-1}\partitioncount_h^k(B_j)}}
	\leq \beta^{-1/2}\sum_{h=1}^H \sum_{j=1}^{\abs{\cC_\sigma}} \pa{ 1 +   \int_{0}^{\partitioncount_h^{K+1}(B_j)} \!\!\frac{\mathrm{d}z}{\sqrt{1+\kernelfunc(4)\beta^{-1} z}}} \mbox{ by Lemma \ref{lemma:aux-sum-of-bonuses}} \\
	& \leq \beta^{-1/2} H \abs{\cC_\sigma} + \frac{2\beta^{1/2}}{\kernelfunc(4)}\sum_{h=1}^H \sum_{j=1}^{\abs{\cC_\sigma}} \sqrt{1 + \kernelfunc(4)\beta^{-1} \partitioncount_h^{K+1}(B_j)} \\
	& \leq \beta^{-1/2} H \abs{\cC_\sigma} + \frac{2\beta^{1/2}}{\kernelfunc(4)} \sum_{h=1}^H \sqrt{ \abs{\cC_\sigma} } \sqrt{ \abs{\cC_\sigma} +  \kernelfunc(4)\beta^{-1} K} \quad \mbox{ by Cauchy-Schwarz inequality}\\
	& \leq H \pa{ \beta^{-1/2} + \frac{2\beta^{1/2}}{\kernelfunc(4)}}\abs{\cC_\sigma} + \frac{2H}{\kernelfunc(4)}\sqrt{ \kernelfunc(4)\abs{\cC_\sigma} K} \lesssim  H \sigmacov  + H \sqrt{\sigmacov K} \,.
	\end{align*}
	
	\paragraph{Bounding the sum of the second order terms}

	\begin{align*}
	& \sum_{k=1}^K\sum_{h=1}^H \frac{1}{\gencount_h^k(\tx_h^k,\ta_h^k)}\indic{\dist{(\tx_h^k, \ta_h^k), (x_h^k, a_h^k)} \leq 2\sigma} \\
	& = \sum_{k=1}^K\sum_{h=1}^H \sum_{j=1}^{\abs{\cC_\sigma}} \frac{1}{\gencount_h^k(\tx_h^k,\ta_h^k)} \indic{\dist{(\tx_h^k, \ta_h^k), (x_h^k, a_h^k)} \leq 2\sigma} \indic{(x_h^k, a_h^k) \in B_j} \\
	& \leq \beta^{-1} \sum_{k=1}^K\sum_{h=1}^H \sum_{j=1}^{\abs{\cC_\sigma}} \frac{1}{1 + \kernelfunc(4)\beta^{-1}\partitioncount_h^k(B_j)} \indic{\dist{(\tx_h^k, \ta_h^k), (x_h^k, a_h^k)} \leq 2\sigma} \indic{(x_h^k, a_h^k) \in B_j} \\
	& \leq \beta^{-1}\sum_{h=1}^H \sum_{j=1}^{\abs{\cC_\sigma}}\sum_{k=1}^K \frac{ \indic{(x_h^k, a_h^k) \in B_j}}{1 + \kernelfunc(4)\beta^{-1}\partitioncount_h^k(B_j)}
	\leq \beta^{-1}\sum_{h=1}^H \sum_{j=1}^{\abs{\cC_\sigma}} \pa{ 1 +   \int_{0}^{\partitioncount_h^{K+1}(B_j)} \frac{\mathrm{d}z}{1+\kernelfunc(4)\beta^{-1} z}} \quad \mbox{ by Lemma \ref{lemma:aux-sum-of-bonuses}} \\
	& \leq \beta^{-1} H \abs{\cC_\sigma} + \frac{1}{\kernelfunc(4)}\sum_{h=1}^H \sum_{j=1}^{\abs{\cC_\sigma}} \log\pa{1 + \kernelfunc(4)\beta^{-1} \partitioncount_h^{K+1}(B_j)} \\
	& \leq \beta^{-1} H \abs{\cC_\sigma} + \frac{1}{\kernelfunc(4)} \sum_{h=1}^H \abs{\cC_\sigma} \log \pa{  \frac{\sum_{j=1}^{\abs{\cC_\sigma}} \pa{1 +  \kernelfunc(4)\beta^{-1} \partitioncount_h^{K+1}(B_j) }}{\abs{\cC_\sigma} }  }  \quad \mbox{ by Jensen's inequality}\\
	& \leq \beta^{-1} H \abs{\cC_\sigma} +  \frac{1}{\kernelfunc(4)} H \abs{\cC_\sigma}\log \pa{  1 + \frac{1 +  \kernelfunc(4)\beta^{-1} K }{\abs{\cC_\sigma} }  } 
	\lesssim H \sigmacov\,.
	\end{align*}
\end{proof}

\begin{ftheorem}
	\label{theorem:regret-bound}
	With probability at least $1-\delta$, the regret of $\lipucbvi$ is bounded as
	\begin{align*}
		\regret(K) \lesssim  &   H^2 \sqrt{\sigmacov K}
		+ H^3 \sigmacov\Xsigmacov 
		+ H^{3/2}\sqrt{K} +  L_1  KH \sigma + H^2 \sigmacov,
	\end{align*}	
	where $\sigmacov$ and $\Xsigmacov$ are the $\sigma$-covering numbers of $(\stateactionspace,\distfunc)$ and $(\statespace, \Sdistfunc)$, respectively.
\end{ftheorem}

\begin{proof}
	The result follows from propositions \ref{prop:regret-aux-result} and \ref{prop:bounding-sum-of-bonuses} and from Hoeffding-Azuma's inequality, which ensures that the term $\sum_{k=1}^K\sum_{h=1}^H (1+1/H)^{H} \widetilde{\xi}_{h+1}^k$ is bounded by $$(\sqrt{8e^2H^2\log(2/\delta)})\sqrt{KH}$$ with probability at least $1-\delta/2$.
\end{proof}

\subsection{Proof of Corollary \ref{corollary:main-text-corollary}}
\label{sec:app-proof-corollary}

Assumption \ref{assumption:metric-state-space} states that $ \dist{(x, a), (x', a')} = \Sdist{x, x'} + \Adist{a, a'}$, which implies that $\Xsigmacov \leq \sigmacov$. Using Theorem \ref{theorem:regret-bound-restated} and the fact that $\sigmacov = \BigO{\sigma^{-d}}$, we obtain $\regret(K) = \BigOtilde{H^2 \sigma^{-d/2}\sqrt{K} + H^3 \sigma^{-2d} + H K \sigma }$. Taking $\sigma = (1/K)^{1/(2d+1)}$, we see that the regret is $\BigOtilde{ H^2 K^{\frac{3d+1}{4d+2}} + H^3K^{\frac{2d}{2d+1}}}$. The fact that $(3d+1)/(4d+2) \leq 2d/(2d+1)$ for $d \geq 1$ allows us to conclude.

\section{Remarks \& regret bounds in different settings} \label{app:variants}

\subsection{Improved regret for stationary MDPs}

The regret bound of \lipucbvi can be improved if the MDP is stationary, \ie $P_1=\ldots=P_H$ and $r_1=\ldots=r_H$.
Let $t = kh$ be the \emph{total time} at step $h$ of episode $k$, and now we index by $t$ all the quantities that were indexed by $(k,h)$, \eg  $w_t(x,a) = w_h^k(x,a)$. In the stationary case, the rewards and transitions estimates become
\begin{align*}
\widehat{\transition}_t(y|x, a) \eqdef  \frac{1}{\gencount_t (x, a)}\sum_{t'=1}^{t-1} w_{t'}(x,a)  \dirac{x_{t'+1}}(y) \quad \mbox{ and } \quad \widehat{\reward}_t(x, a) \eqdef  \frac{1}{\gencount_t (x, a)} \sum_{t'=1}^{t-1} w_{t'}(x,a) \reward_{t'} \,,
\end{align*}
respectively, where we redefine the generalized counts as
\begin{align*}
\gencount_t (x, a) \eqdef \beta + \sum_{t'=1}^{t-1}w_{t'}(x,a).
\end{align*}

The proofs of the concentration results and of the regret bound remain valid, in particular Proposition \ref{prop:regret-aux-result}, up to minor changes in the constants $\bonusvarP(k,h), \bonusbiasP(k,h), \bonusvarR(k,h), \bonusbiasR(k,h),  \bernvar(k, h)$  and $\bernbias(k, h)$ . However, the bounds presented in Proposition \ref{prop:bounding-sum-of-bonuses} can be improved to obtain a better regret bound in terms of the horizon $H$. Consider the sets $B_j$ introduced in the proof of Proposition \ref{prop:bounding-sum-of-bonuses} and let
\begin{align*}
	\partitioncount_t(B_j) \eqdef \sum_{t'=1}^{t-1}\indic{(x_t, a_t) \in B_j}.
\end{align*}
As we did in the proof Proposition \ref{prop:bounding-sum-of-bonuses}, we can show that
$\gencount_t(\tilde{x}_t,\tilde{a}_t) \geq \beta + \kernelfunc(4)\partitioncount_t(B_j)$
if $(x_t, a_t) \in B_j $ and $\dist{(\tilde{x}_t, \tilde{a}_t), (x_t, a_t)} \leq 2\sigma$.
The sum of the first order terms $\sum_{t} 1/\sqrt{\gencount_t(\tilde{x}_t,\tilde{a}_t)}$ is now bounded as
\begin{align*}
& \sum_{t=1}^{KH} \sqrt{\frac{1}{\gencount_t(\tilde{x}_t,\tilde{a}_t)}}\indic{\dist{(\tilde{x}_t, \tilde{a}_t), (x_t, a_t)} \leq 2\sigma} \\
& \leq \beta^{-1} \sum_{j=1}^{\abs{\cC_\sigma}}\sum_{t=1}^{KH} \frac{ \indic{(x_t, a_t) \in B_j}}{\sqrt{1 + \kernelfunc(4)\beta^{-1}\partitioncount_t(B_j)}}
\leq \beta^{-1} \sum_{j=1}^{\abs{\cC_\sigma}} \pa{ 1 +   \int_{0}^{\partitioncount_{KH+1}(B_j)} \frac{\mathrm{d}z}{\sqrt{1+\kernelfunc(4)\beta^{-1} z}}} \quad \mbox{ by Lemma \ref{lemma:aux-sum-of-bonuses}} \\
& \leq \beta^{-1} \abs{\cC_\sigma} + \frac{2}{\kernelfunc(4)} \sum_{j=1}^{\abs{\cC_\sigma}} \sqrt{1 + \kernelfunc(4)\beta^{-1} \partitioncount_{KH+1}(B_j)} \\
& \leq \beta^{-1} \abs{\cC_\sigma} + \frac{2}{\kernelfunc(1)} \sqrt{ \abs{\cC_\sigma} } \sqrt{ \abs{\cC_\sigma} +  \kernelfunc(4)\beta^{-1} KH} \quad \mbox{ by Cauchy-Schwarz inequality}\\
& \leq \pa{ \beta^{-1} + \frac{2}{\kernelfunc(4)}}\abs{\cC_\sigma} + \frac{2}{\kernelfunc(4)}\sqrt{ \kernelfunc(4)\beta^{-1}\abs{\cC_\sigma} H K} \\
& = \BigO{ \abs{\cC_\sigma} + \sqrt{\abs{\cC_\sigma} HK} }.
\end{align*}
When compared to the non-stationary case, where the corresponding sum is bounded by $\BigO{ H \abs{\cC_\sigma} + H \sqrt{\abs{\cC_\sigma} K} }$, we gain a factor of $\sqrt{H}$ in the term multiplying $\sqrt{K}$ and a factor of $H$ in the term multiplying $\abs{\cC_\sigma}$.

Similarly, the sum of the second order terms $\sum_{t} 1/\gencount_t(\tilde{x}_t,\tilde{a}_t)$ is now bounded as

\begin{align*}
\sum_{t=1}^{KH} \frac{1}{\gencount_t(\tilde{x}_t,\tilde{a}_t)}\indic{\dist{(\tilde{x}_t, \tilde{a}_t), (x_t, a_t)} \leq 2\sigma}
& \leq \beta^{-1} \abs{\cC_\sigma} +  \frac{1}{\kernelfunc(4)} \abs{\cC_\sigma}\log \pa{  1 + \frac{1 +  \kernelfunc(4)\beta^{-1} KH }{\abs{\cC_\sigma} }  }\\
& = \BigOtilde{\abs{\cC_\sigma}} \,.
\end{align*}
In the non-stationary case, the corresponding sum is bounded by $\BigOtilde{H \abs{\cC_\sigma}}$, thus we gain a factor of $H$.

Hence, if the MDP is stationary, we obtain a regret bound of
\begin{align*}
\regret_{\mathrm{stationary}}(K) = \BigOtilde{ H^{3/2} \sqrt{\abs{\cC_\sigma}K} + L_1 H K \sigma + H^2 \abs{\cC_\sigma}^2}
\end{align*}
which is $\BigOtilde{H^2 K^{\max\left(\frac{1}{2}, \frac{2d}{2d+1}\right)}}$ by taking $\sigma = (1/K)^{1/(2d+1)}$.

\subsubsection{Important remark}

Computationally, in order to achieve this improved regret for \lipucbvi, every time a new transition and a new reward are observed at a step $h$, the estimates $\widehat{\transition}_t(y|x, a)$ and $\widehat{\reward}_t(x, a)$ need to be updated, and the optimistic $Q$-functions need to be recomputed through backward induction, which increases the computational complexity by a factor of $H$.

The UCBVI-CH algorithm of \cite{Azar2017} in the tabular setting for stationary MDPs also suffers from this problem. If the optimistic $Q$-function is not recomputed at every step $h$, its regret is $\BigOtilde{H^{3/2}\sqrt{XAK} + H^3 X^2 A}$ and not $\BigOtilde{H^{3/2}\sqrt{XAK} + H^2 X^2 A}$, where $X$ is the number of states, as claimed in their paper. To see why, let's analyze its second order term, which is $\BigO{H^2 X \sum_{k,h} 1/N_k(x_h^k,a_h^k)}$\footnote{See page 7 of \cite{Azar2017}.}, where $N_k(x,a)$ is the number of visits to $(x, a)$ \emph{before} episode $k$, \ie

$$N_k(x,a) = \max\pa{1, \sum_{h=1}^H\sum_{s=1}^{k-1} \indic{(x_h^s, a_h^s) = (x, a)}}.$$

If $K \geq XA$, and if $N_k(x,a)$ is updated \textbf{only at the end} of each episode, we can show that there exists a sequence $(x_h^k, a_h^k)$ such that the sum $\sum_{k,h} 1/N_k(x_h^k,a_h^k)$ is greater than $HXA$. Let $(x_k, a_k)_{k\in[XA]}$ be $XA$ \emph{distinct} state-action pairs, and take the sequence $(x_h^k, a_h^k)_{h\in[H], k\in[XA]}$ such that $(x_h^k, a_h^k) = (x_k, a_k)$. That is, in each of the $XA$ episodes, the algorithm visits, in each of the $H$ steps, \emph{only one} state-action pair that \emph{has never been visited before}. Since $N_k(x,a)$ is updated only at the end of the episodes, we have $N_k(x_h^k,a_h^k) = 1$ for all $h \in[H]$ and $k \in [XA]$, with this choice of $(x_h^k,a_h^k)_{h, k}$. Hence,
\begin{align*}
	H^2 X \sum_{k=1}^{XA}\sum_{h=1}^H \frac{1}{N_k(x_h^k,a_h^k)} = H^2 X \sum_{k=1}^{XA}\sum_{h=1}^H 1 = H^3 X^2 A.
\end{align*}

 Consequently, the sum of second order term is lower bounded (in a worst case sense) by $H^3 X^2 A$ and cannot be $\BigOtilde{H^2 X^2 A}$ as claimed by \cite{Azar2017}, since their bound \emph{must hold for any possible sequence} $(x_h^k,a_h^k)_{h, k}$. An application of Lemma \ref{lemma:aux-sum-of-bonuses} with $c = H$ can be used to show that the second order term is indeed $\BigOtilde{H^3 X^2 A}$ when updates are done at the end of the episodes only. 
 
 To gain a factor of $H$ (\ie have $\BigOtilde{H^2 X^2 A}$ as second order term), one solution is to update the counts $N_k(x_h^k,a_h^k)$ every time a new state-action pair is observed, and recompute the optimistic $Q$-function. Another solution is to recompute it every time the number of visits of the current state-action pair is \emph{doubled}, as done by \cite{jaksch2010near} in the average-reward setting.

The efficient version of our algorithm, \lipucbvigreedy, does not suffer from this increased computational complexity in the stationary case. This is due to the fact that the value functions are updated in real time, and there is no need to run a backward induction every time a new transition is observed. Hence, in the stationary case, \lipucbvigreedy has a regret bound that is $H$ times smaller than in the non-stationary case, \emph{without} an increase in the computational complexity.

\subsection{Dependence on the Lipschitz constant \& regularity \wrt the total variation distance}
Notice that the regret bound of \lipucbvi has a linear dependency on $L_1$ that appears in the bias term $L_1 H K \sigma$:
\begin{align*}
\regret(K) \leq \BigOtilde{H^2 \sqrt{\abs{\cC_\sigma}K} + L_1 H K \sigma + H^3 \sigmacov\Xsigmacov + H^2\abs{\cC_\sigma} }\,.
\end{align*}
As long as the Lipschitz constant $L_1 = \sum_{h= 1}^H \lambda_r \lambda_p^{H-h}$ is $\BigO{H}$ or $\BigO{H^2}$, our regret bound has no additional dependency on $H$. However, if $\lambda_p > 1$, the constant $L_1 $ can be exponential in $H$.
This issue is caused by the smoothness of the MDP and not by algorithmic design. With minor modifications to our proof, we could also consider that the transitions are Lipschitz with respect to the total variation distance, in which case $L_1$ would always be $\BigO{H}$ and the regret of \lipucbvi would remain $\BigOtilde{H^3 K^{\max\left(\frac{1}{2}, \frac{2d}{2d+1}\right)}}$ by taking $\sigma = (1/K)^{1/(2d+1)}$.
The regret bounds of other algorithms for Lipschitz MDPs also depend on the Lipschitz constant, which always appears in a bias term (\eg \cite{ortner2012online}).

In addition, the value $L_h = \sum_{h'= h}^H \lambda_r \lambda_p^{H-h'}$ represents simply an upper bound on the Lipschitz constant of the $Q$-function $Q_h^*$. If the functions $Q_h^*$ for $h\in[H]$ are $\widetilde{L}_h$-Lipschitz with $\widetilde{L}_h$ known and such that $\widetilde{L}_h < L_h$, \lipucbvi could exploit the knowledge of $\widetilde{L}_h$ and use it instead of $L_h$, which would also improve the regret bound. For instance, if all rewards functions $r_h$ are $0$ except for $r_H$, we could use $\widetilde{L}_h = \lambda_r$, the Lipscthiz constant of $r_H$, which is independent of $H$.

\section{Efficient implementation}
\label{app:efficient_implementation}

In this Appendix, following \cite{efroni2019tight}, we show that if we only apply the optimistic Bellman operator once instead of doing a complete value iteration we obtain almost the same guaranties as for Algorithm~\ref{alg:lipucbvi} but with a large improvement in computational complexity. Indeed, the time complexity of each episode $k$ is reduced from $O(k^2)$ to $O(k)$. This complexity is comparable to other model-based algorithm in structured MDPs~\citep[\eg][]{Jin2019}.

The algorithm goes as follows. Assume we are at episode $k$ at step $h$ at state $x_h^k$. To compute the next action we will apply the optimistic Bellman operator to the previous value function. That is, for all $a\in\actionspace$ we compute the upper bounds on the $Q$-value based on a kernel estimator:
\[\widetilde{Q}_h^k(x_h^k, a) = \widehat{\reward}_h^k(x, a) + \widehat{P}_h^k V_{h+1}^{k}(x, a) + \bonus_h^k(x, a)\,.\]
Then we act greedily
\[
a_h^k = \argmax_{a\in\actionspace} \widetilde{Q}_h^k(x_h^k, a) \,,
\]
and define a new optimistic target $ \tV_h^k(x_h^k) = \min\big(H-h+1,\widetilde{Q}_h^k(x_h^k, a_h^k)\big)$ for the value function at state $x_h^k$. Then we build an optimistic value function $V_h^k$ by interpolating the previous optimistic target and the new one we just defined
\[
\forall x, V_h^{k+1}(x) =\min\!\!\left( \min_{s\in[k-1]}\left[ V_h^{k}(x_h^s) + L_h\Sdist{x,x_h^s} \right], \tV_h^k(x_h^k) + L_h\Sdist{x,x_h^k}  \right)\,.
\]
The complete procedure is detailed in Algorithm~\ref{alg:lipucbvigreedy.maintext}.

\begin{algorithm}[t]
	\caption{\lipucbvigreedy}
	\label{alg:lipucbvigreedy.maintext}
	\begin{small}
		\begin{algorithmic}
			\STATE {\bfseries Input:} global parameters $K, H, \delta, \lambda_r, \lambda_p, \sigma, \beta$
			\STATE initialize $\data_h = \emptyset$ and $V_h^1(x) = H-h+1$, for all $h \in [H]$
			\FOR{episode $k = 1, \ldots, K$}
				\STATE get initial state $x_1^k$ 
				\FOR{step $h=1, \ldots, H$}
					\STATE \textcolor{darkgreen}{// define, for all $a$: }
					\STATE $\widetilde{Q}_h^k(x_h^k, a) = \sum_{s=1}^{k-1} \widetilde{w}_h^s(x_h^k, a) \pa{r_h^s + V_{h+1}^k(x_{h+1}^s)} + \bonus_h^k(x_h^k, a)$ 
					\STATE execute $a_h^k = \argmax_a \widetilde{Q}_h^k(x_h^k, a)$
					\STATE observe $r_h^k$ and $x_{h+1}^k$
					\STATE $\tV_h^k(x_h^k) = \min\pa{H-h+1, \max_{a\in\actionspace}\widetilde{Q}_h^k(x_h^k,a)}$ 
					\STATE \textcolor{darkgreen}{// interpolate: define $V_h^{k+1}$ for all $x \in \data_h$ as}
					\STATE $V_h^{k+1}(x) =\min\Big( \min_{s\in[k-1]}\left[ V_h^{k}(x_h^s) + L_h\Sdist{x,x_h^s} \right],\tV_h^k(x_h^k) + L_h\Sdist{x,x_h^k}  \Big)$
					\STATE add $(x_h^k, a_h^k, x_{h+1}^k, r_h^k)$ to $\data_h$
				\ENDFOR
			\ENDFOR 
		\end{algorithmic}
	\end{small}
\end{algorithm}

\begin{proposition}[Optimism]
	\label{prop:optimism_greedy}
	In the event $\favevent$, whose probability is greater than $1-\delta$, we have:
	\begin{align*}
	\forall (k, h), \forall x, \; V_h^k(x) \geq V_h^*(x) \text{ and } V_h^k(x) \geq V_h^{k+1}(x)\,.
	\end{align*}
\end{proposition}
\begin{proof}

	To show that $V_h^k(x) \geq V_h^{k+1}(x)$, notice that
	\[
	\forall x, V_h^{k+1}(x) =\min\pa{ V_h^k(x), \tV_h^k(x_h^k) + L_h\Sdist{x,x_h^k}  } \leq V_h^k(x)
	\]
	since, by definition, $ V_h^k(x) = \min_{s\in[k-1]}\left[ V_h^{k}(x_h^s) + L_h\Sdist{x,x_h^s} \right]$.

	To show that $V_h^k(x) \geq V_h^*(x)$, we proceed by induction on $k$. For $k=1$, $V_h^k(x) = H-h \geq V_h^*(x)$ for all $x$ and $h$.

	Now, assume that $V_h^{k-1} \geq V_h^*$ for all $h$. As in the proof of Proposition \ref{prop:optimism}, we prove that $V_h^k \geq V_h^*$ by induction on $h$.
	For $h = H+1$, $V_h^k(x) = V_h^*(x) = 0$ for all $x$. Now, assume that
	$V_{h+1}^k(x) \geq V_{h+1}^*(x)$ for all $x$.
	We have, for all $(x, a)$,
	\begin{align*}
		\widetilde{Q}_h^k(x, a) & = \widehat{\reward}_h^k(x, a) + \widehat{P}_h^k V_{h+1}^{k}(x, a) + \bonus_h^k(x, a) \\
		& \geq \widehat{\reward}_h^k(x, a) + \widehat{P}_h^k V_{h+1}^*(x, a) + \bonus_h^k(x, a) \quad \mbox{ by induction hypothesis on } h \\
		& \geq \reward_h(x, a) + P_h V_{h+1}^{*}(x, a) = Q_h^*(x, a)  \quad \mbox{ in } \favevent
	\end{align*}
	which implies that $\tV_h^k(x_h^k) \geq V_h^{*}(x_h^k)$ and, consequently,
	\begin{align*}
		& \tV_h^k(x_h^k)  + L_h\Sdist{x,x_h^k} \geq V_h^{*}(x_h^k)  + L_h\Sdist{x,x_h^k} \geq V_h^{*}(x) \\
		& \implies V_h^{k}(x) =\min\pa{ V_h^{k-1}(x), \tV_h^k(x_h^k) + L_h\Sdist{x,x_h^k}  } \geq V_h^*(x) \quad \mbox{ by induction hypothesis on } k
	\end{align*}
	and we used the fact that $V_h^{*}$ is $L_h$-Lipschitz.
\end{proof}
\begin{fproposition}
	\label{prop:regret-aux-result-greedy}
	With probability at least $1-\delta$, the regret of $\lipucbvigreedy$ is bounded as
	\begin{align*}
		\regret(K) \lesssim  &   H^2 \sqrt{\sigmacov K}
		+ H^3 \sigmacov\Xsigmacov 
		+ H^{3/2}\sqrt{K} +  L_1  KH \sigma + H^2 \sigmacov + H^2 \Xsigmacov,
	\end{align*}
	where $\sigmacov$ and $\Xsigmacov$ are the $\sigma$-covering numbers of $(\stateactionspace,\distfunc)$ and $(\statespace, \distfunc)$, respectively.
\end{fproposition}

\begin{proof}
	On $\favevent$, we have
	\begin{align*}
	\tdelta_h^k & \eqdef V_h^{k+1}(x_h^k) - V_h^{\pi_k}(x_h^k) 
	\leq   V_h^{k}(x_h^k) - V_h^{\pi_k}(x_h^k)\\
    &\leq \tV_h^{k}(x_h^k) - V_h^{\pi_k}(x_h^k)
	\leq \tQ_h^k(x_h^k,a_h^k) - Q_h^{\pi_k}(x_h^k,a_h^k) \\
	\end{align*}
	From this point we can follow the proof of Proposition~\ref{prop:regret-aux-result} to obtain

	\begin{align*}
		\tdelta_h^k & \lesssim
		\pa{1+\frac{1}{H}}\pa{\delta_{h+1}^k + \xi_{h+1}^k}
		+  L_1 \dist{(\tx_h^k, \ta_h^k), (x_h^k, a_h^k)}		
		+  \sqrt{\frac{H^2}{\gencount_h^k(\tx_h^k, \ta_h^k)}}
		+  \frac{H^2\Xsigmacov}{\gencount_h^k(\tx_h^k, \ta_h^k)} 
		+ L_1\sigma \\
		& \lesssim  \pa{1+\frac{1}{H}}\pa{\tdelta_{h+1}^k + \pa{V_{h+1}^{k}-V_{h+1}^{k+1}}(x_{h+1}^k) + \xi_{h+1}^k}
		+  L_1 \dist{(\tx_h^k, \ta_h^k), (x_h^k, a_h^k)} \\		
		& +  \sqrt{\frac{H^2}{\gencount_h^k(\tx_h^k, \ta_h^k)}}
		+  \frac{H^2\Xsigmacov}{\gencount_h^k(\tx_h^k, \ta_h^k)} 
		+ L_1\sigma
	\end{align*}

On $\favevent$, using that $V^*_h\leq V^{k+1}_h$ and  the same arguments as in equations \eqref{eq:trick-delta-one} and \eqref{eq:trick-delta-two} in Proposition~\ref{prop:regret-aux-result} (which can be used since $V_{h+1}^{k} \geq V_{h+1}^{k+1}$), we obtain
\begin{align*}
\regret(K) & \leq \sum_{k=1}^K \tdelta_1^k \\
  &  \lesssim H^2 \sigmacov + L_1  KH \sigma + \sum_{k=1}^K\sum_{h=1}^H \pa{1+\frac{1}{H}}^{h} \xi_{h+1}^k \\
&  + \sum_{k=1}^K\sum_{h=1}^H
\pa{  
	\frac{H}{\sqrt{\gencount_h^k(\tx_h^k, \ta_h^k)}} + 
	\frac{H^2 \Xsigmacov}{\gencount_h^k(\tx_h^k, \ta_h^k)}
}
\indic{\dist{(\tx_h^k, \ta_h^k), (x_h^k, a_h^k)} \leq 2 \sigma}  \\
& +  \sum_{k=1}^K\sum_{h=1}^H  \pa{1+\frac{1}{H}}^h \pa{V_{h+1}^{k}-V_{h+1}^{k+1}}(x_{h+1}^k)
\end{align*}
This bound differs only by the last additive term above from the bound given in Proposition~\ref{prop:regret-aux-result}. Thus we just need to handle this sum and rely on the previous analysis to upper bound the other terms. We consider the following partition of the state space: 
\begin{definition}
	\label{def:partition-for-regret-bound-greedy}
	Let $\tcC_\sigma$ be a $\sigma$-covering of $\statespace$. We write $\tcC_\sigma \eqdef \braces{x_j, j\in [\abs{\cC_\sigma}]}$. For each $x_j \in \tcC_\sigma$, we define the set $B_j \subset \statespace$ as the set of points in $\statespace$ whose nearest neighbor in $\tcC_\sigma$ is $x_j$, with ties broken arbitrarily, such that $\braces{B_j}_{j\in [\abs{\cC_\sigma}]}$ form a partition of $\statespace$.
\end{definition}
Using the fact that the $V_h^k$ are point-wise non-increasing we can transform the last sum in the previous inequality in a telescopic sum
\begin{align*}
\sum_{k=1}^K\sum_{h=1}^H  \pa{1+\frac{1}{H}}^h \pa{V_{h+1}^{k}-V_{h+1}^{k+1}}(x_{h+1}^k) &\leq e \sum_{k=1}^K\sum_{h=1}^H \pa{V_{h+1}^{k}-V_{h+1}^{k+1}}(x_{h+1}^k)\\
&\leq e \sum_{j=1}^{\abs{\tcC_\sigma}}\sum_{k=1}^K\sum_{h=1}^H \pa{V_{h+1}^{k}-V_{h+1}^{k+1}}(x_{h+1}^k) \indic{x_{h+1}^k \in B_j}\\
&\leq e \sum_{j=1}^{\abs{\tcC_\sigma}}\sum_{k=1}^K\sum_{h=1}^H \pa{V_{h+1}^{k}-V_{h+1}^{k+1}}(x_j) \indic{x_{h+1}^k \in B_j} \\
&\qquad+2 L_h\Sdist{x_j,x_{h+1}^k} \indic{x_{h+1}^k \in B_j}\\
&\leq e \sum_{j=1}^{\abs{\tcC_\sigma}}\sum_{k=1}^K\sum_{h=1}^H \pa{V_{h+1}^{k}-V_{h+1}^{k+1}}(x_j) + e K \sum_{h=1}^H 2 L_1 \sigma\\
&\leq e H^2 \abs{\tcC_\sigma}+ 2 e \sigma L_1 H K\,,
\end{align*}
where in the third inequality, we used the fact that the function $V_{h+1}^{k}-V_{h+1}^{k+1}$ is $2 L_h$-Lipschitz. Combining the previous inequalities and the proof of Theorem~\ref{theorem:regret-bound}, as explained above, allows us to conclude.
\end{proof}

\section{New concentration inequalities}

In this section we present two new concentration inequalities that control, uniformly over time, the deviation of weighted sums of zero-mean random variables. They both follow from the so-called method of mixtures (e.g.,  \cite{pena2008self}), and can have applications beyond the scope of this work.

\begin{lemma}[Hoeffding type inequality]
	\label{lemma:self-normalized-weighted-sum}
	Consider the sequences of random variables $(w_t)_{t\in\NN^*}$ and $(Y_t)_{t\in\NN^*}$ adapted to a filtration $(\cF_t)_{t\in\NN}$. Assume that, for all $t\geq 1$, $w_t$ is $\cF_{t-1}$ measurable and $\expect{ \exp(\lambda Y_t)\given \cF_{t-1}} \leq \exp(\lambda^2 c^2/2)$ for all $\lambda > 0$.

	Let
	\begin{align*}
	S_t \eqdef \sum_{s=1}^t w_s Y_s \quad \mbox{and} \quad V_t \eqdef \sum_{s=1}^t w_s^2.
	\end{align*}

	Then, for any $\beta > 0$, with probability at least $1-\delta$, for all $t\geq 1$,
	\begin{align*}
	\frac{\abs{S_t}}{\sum_{s=1}^t w_s + \beta} \leq \sqrt{ 2c^2 \left[ \log\pa{\frac{1}{\delta}} + \frac{1}{2}\log\pa{\frac{V_t + \beta}{\beta}}  \right] \frac{V_t + \beta}{\pa{\sum_{s=1}^t w_s+\beta}^2} }\,.
	\end{align*}
	In addition, if $w_s \leq 1$ almost surely for all $s$, we have $V_t \leq \sum_{s=1}^t w_s \leq t$ and the above can be simplified to
	\begin{align*}
	 \frac{\abs{S_t}}{\sum_{s=1}^t w_s + \beta} \leq \sqrt{ 2c^2  \log\pa{\frac{\sqrt{1 + t/\beta}}{\delta}}  \frac{1}{\sum_{s=1}^t w_s+\beta}}\,.
	\end{align*}
\end{lemma}

\begin{proof}
	Let
	\begin{align*}
	M_t^\lambda = \exp\pa{\lambda S_t - \frac{\lambda^2 c^2 V_t}{2}},
	\end{align*}
	with the convention $M_0^{\lambda} = 1$.
	The process $\braces{M_t^\lambda}_{t\geq 0}$ is a supermartingale, since
	\begin{align}
	\expect{M_t^\lambda \given \cF_{t-1}} = \expect{\exp\pa{w_t Y_t - \frac{\lambda^2 c^2 w_t^2}{2}}\given \cF_{t-1}} M_{t-1}^\lambda \leq M_{t-1}^\lambda,
	\end{align}
	which implies that $\expect{M_t^\lambda} \leq  \expect{M_0^\lambda} = 1$. Now, we apply the method of mixtures, as in \citep{pena2008self} see also \citep{Abbasi-Yadkori2011}. We define the supermartingale $M_t$ as
	\begin{align*}
	M_t = \sqrt{\frac{\beta c^2}{2\pi}}\int_{\RR} M_t^\lambda \exp\pa{-\frac{\beta c^2\lambda^2}{2}} \mathrm{d}\lambda = \sqrt{\frac{\beta}{V_t + \beta}}\exp\pa{\frac{S_t^2}{2(V_t + \beta)c^2}}.
	\end{align*}
	The maximal inequality for non-negative supermartingales gives us:
	\begin{align*}
	\prob{\exists t \geq 0: M_t \geq \delta^{-1}} \leq \delta \expect{M_0} = \delta.
	\end{align*}
	Hence, with probability at least $1-\delta$, we have
	\begin{align*}
	\forall t\geq 0, \quad \abs{S_t} \leq \sqrt{2c^2\left[\log(1/\delta) + (1/2) \log((V_t+\beta)/\beta) \right](V_t +\beta)}.
	\end{align*}
	Dividing both sides by $\sum_{s=1}^t w_s + \beta$ gives the result.
\end{proof}

\begin{lemma}[Bernstein type inequality]
	\label{lemma:bernstein-freedman-weighted-sum}
	Consider the sequences of random variables $(w_t)_{t\in\NN^*}$ and $(Y_t)_{t\in\NN^*}$ adapted to a filtration $(\cF_t)_{t\in\NN}$. Let
	\begin{align*}
		S_t \eqdef \sum_{s=1}^t w_s Y_s, \quad V_t \eqdef \sum_{s=1}^t w_s^2\expect{Y_s^2\given \cF_{s-1}} \quad \mbox{and} \quad W_t \eqdef \sum_{s=1}^t w_s\,,
	\end{align*}
	and $h(x) = (x+1) \log(x+1)-x$.
	Assume that, for all $t\geq 1$,
	\begin{itemize}
		\item $w_t$ is $\cF_{t-1}$ measurable,
		\item $\expect{Y_t\given \cF_{t-1}} = 0$,
		\item $w_t \in [0, 1]$ almost surely,
		\item there exists $b > 0$ such that $\abs{Y_t} \leq b$ almost surely.
	\end{itemize}
	Then, we have
	\begin{align*}
		\prob{\exists t\geq 1,   (V_t/b^2+1)h\!\left(\frac{b |S_t|}{V_t+b^2}\right) \geq \log(1/\delta) + \log\big(4e(2t+1)\big) }\leq \delta\,.
	\end{align*}
 The previous inequality can be weakened to obtain a more explicit bound: for all $\beta > 0$, with probability at least $1-\delta$, for all $t\geq 1$,
 \[
 \frac{|S_t|}{\beta + \sum_{s=1}^t w_s}\leq \sqrt{2\log\big(4e(2t+1)/\delta\big) \frac{V_t+b^2}{\left(\beta + \sum_{s=1}^t w_s\right)^2}} + \frac{2b}{3}\frac{\log\!\big(4e(2t+1)/\delta\big)}{\beta + \sum_{s=1}^t w_s}\,.
 \]
\end{lemma}
\begin{proof}
	By homogeneity we can assume that $b=1$ to prove the first part. First note that for all $\lambda > 0$,
	\[
	 e^{\lambda w_t Y_t} -\lambda w_t Y_t -1 \leq (w_t Y_t)^2(e^{\lambda}-\lambda -1)\,,
	\]
	because the function $y \to (e^y-y-1)/y^2$ (extended by continuity at zero) is non-decreasing. Taking the expectation yields
	\[
	\expect{e^{\lambda w_t Y_t}|\cF_{t-1}} -1 \leq w_t^2\expect{ Y_t^2 | \cF_{t-1}}(e^{\lambda}-\lambda -1)\,,
	\]
	thus using $y+1\leq e^y$ we get
	\[
	\expect{e^{\lambda (w_t Y_t)}|\cF_{t-1}} \leq e^{w_t^2\expect{ Y_t^2 | \cF_{t-1}}(e^{\lambda}-\lambda -1)}\,.
	\]
	We just proved that the following quantity is a supermartingale with respect to the filtration $(\cF_t)_{t\geq 0}$,
	\[
	M_t^{\lambda,+} = e^{\lambda(S_t+V_t) - V_t (e^\lambda - 1)}\,.
	\]
	Similarly, using that the same inequality holds for $-X_t$, we have
	\[
	\expect{e^{-\lambda w_t Y_t}|\cF_{n-1}} \leq e^{w_t^2\expect{ Y_t^2 | \cF_{t-1}}(e^{\lambda}-\lambda -1)}\,,
	\]
	thus, we can also define the supermartingale
	\[
	M_t^{\lambda,-} = e^{\lambda(-S_t+V_t) - V_t (e^\lambda - 1)}\,.
	\]
	We now choose the prior over $\lambda_x = \log(x+1)$ with $x\sim \Exponential(1)$, and consider the (mixture) supermartingale
	\[
	M_t = \frac{1}{2} \int_{0}^{+\infty} e^{\lambda_x(S_t+V_t) - V_t (e^\lambda_x - 1)} e^{-x} \mathrm{d} x +\frac{1}{2} \int_{0}^{+\infty} e^{\lambda_x(-S_t+V_t) - V_n (e^\lambda_x - 1)} e^{-x} \mathrm{d} x \,.
	\]
	Note that by construction it holds $\expect{M_t} \leq 1$. We will apply the method of mixtures to that super martingale thus we need to lower bound it with the quantity of interest. To this aim we will we will lower bound the integral by the one only around the maximum of the integrand. Using the change of variable $\lambda = \log(1+x)$, we obtain
	\begin{align*}
	  M_t &\geq \frac{1}{2} \int_{0}^{+\infty} e^{\lambda_x (|S_t|+V_t) - V_t (e^{\lambda_x} - 1)} e^{-x} \mathrm{d} x
	  \geq \frac{1}{2} \int_{0}^{+\infty}  e^{\lambda(|S_t| + V_t +1) - (V_t +1)(e^\lambda-1)} \mathrm{d} \lambda \\
	  &\geq  \frac{1}{2} \int_{\log\big(|S_t|/(V_t+1) + 1\big)}^{\log\big(|S_t|/(V_t+1) + 1 + 1/(V_t+1)\big)}  e^{\lambda(|S_t| + V_t +1) - (V_t +1)(e^\lambda-1)} \mathrm{d} \lambda\\
	  &\geq \frac{1}{2} \int_{\log\big(|S_t|/(V_t+1) + 1\big)}^{\log\!\big(|S_t|/(V_t+1) + 1 + 1/(V_t+1)\big)}  e^{\log\big(|S_t|/(V_t+1) + 1\big)(|S_t| + V_t +1) - |S_t| - 1} \mathrm{d} \lambda\\
	  &= \frac{1}{2e}e^{(V_t+1)h\big(|S_t|/(V_t+1)\big)} \log\!\!\left( 1+\frac{1}{|S_t|+V_t+1}\right)
	  \geq \frac{1}{4e(2 t+1)}e^{(V_t+1)h\big(|S_t|/(V_t+1)\big)}\,,
	\end{align*}
	where in the last line we used $\log(1+1/x)\geq 1/(2x)$ for $x \geq 1$ and the trivial bounds $|S_t| \leq 1$, $V_t \leq t$. The method of mixtures, see \citep{pena2008self}, allows us to conclude for the first inequality of the lemma. The second inequality is a straightforward consequence of the previous one. Indeed, using that (see Exercise 2.8 of \cite{boucheron2013concentration}) for $x\geq 0$
	  \[
	  h(x) \geq \frac{x^2}{2(1+x/3)}\,,
	  \]
	we get
	\[
	\frac{|S_t|/b}{V_t/b^2+1} \leq \sqrt{\frac{2\log\big(4e(2t +1)/\delta\big)}{V_t/b^2+1}} + \frac{2}{3}\frac{\log\big(4e(2t+1)/\delta\big)}{V_t/b^2+1}\,.
	\]
	Dividing by $\beta+\sum_{s=1}^t w_s$ and multiplying by $b(V_t/b^2+1)$ the previous inequality allows us to conclude.
\end{proof}

\section{Auxiliary results}

\subsection{Proof of Lemma~\ref{lemma:q_function_is_lipschitz}} \label{proof:q_function_is_lipschitz}

In this section, we prove that the optimal $Q$-functions $Q_h$ are Lipschitz continuous.

\begin{lemma}[Value functions are Lipschitz continuous]
	\label{lemma:value-funcs-are-liptschitz}
	Under assumption \ref{assumption:lipschitz-rewards-and-transitions} we have:
	\begin{align*}
	\forall (x, a, x', a'), \; \forall h \in [H], \quad \abs{Q_h^*(x, a) - Q_h^*(x', a')} \leq L_h \dist{(x,a), (x', a')}
	\end{align*}
	where $L_h \eqdef \sum_{h'= h}^H \lambda_r \lambda_p^{H-h'}$.
\end{lemma}

\begin{proof}
	We proceed by induction. For $h = H$, $Q_h^*(x, a) = r(x, a)$ and the statement is true, since $r$ is $\lambda_r$-Lipschitz. Now, assume that it is true for $h+1$ and let's prove it for $h$.

	First, we note that $V_{h+1}^*(x)$ is Lipschitz by the induction hypothesis:
	\begin{align*}
	V_{h+1}^*(x) - V_{h+1}^*(x') & =  \max_a Q_{h+1}^*(x, a) -  \max_a Q_{h+1}^*(x', a)  \leq \max_a \pa{Q_{h+1}^*(x, a) - Q_{h+1}^*(x', a)} \\
	& \leq \max_a \sum_{h'= h+1}^H \lambda_r \lambda_p^{H-h'} \dist{(x, a), (x',a)} = \sum_{h'= h+1}^H \lambda_r \lambda_p^{H-h'} \Sdist{x, x'},
	\end{align*}
	where, in the last equality, we used  the fact that $ \dist{(x, a), (x', a')} = \Sdist{x, x'} + \Adist{a, a'}$ by Assumption \ref{assumption:metric-state-space}.

	By applying the same argument and inverting the roles of $x$ and $x'$, we obtain
	\begin{align*}
	\abs{V_{h+1}^*(x) - V_{h+1}^*(x')} \leq  \sum_{h'= h+1}^H \lambda_r \lambda_p^{H-h'} \Sdist{x, x'}.
	\end{align*}
	Now, we have
	\begin{align*}
	Q_h^*(x, a) - Q_h^*(x', a') &\leq \lambda_r\dist{(x, a), (x',a')} + \int_{\statespace}V_{h+1}^*(y)(P_h(\mathrm{d}y|x,a) - P_h(\mathrm{d}y|x',a')) \\
	&\leq \lambda_r\dist{(x, a), (x',a')} + L_{h+1}\int_{\statespace}\frac{V_{h+1}^*(y)}{L_{h+1}}(P_h(\mathrm{d}y|x,a) - P_h(\mathrm{d}y|x',a')) \\
	& \leq \left[ \lambda_r + \lambda_p \sum_{h'= h+1}^H \lambda_r \lambda_p^{H-h'} \right]\dist{(x, a), (x',a')} = \sum_{h'= h}^H \lambda_r \lambda_p^{H-h'}\dist{(x, a), (x',a')}
	\end{align*}
	where, in last inequality, we use fact that $V_{h+1}^*/L_{h+1}$ is $1$-Lipschitz, the definition of the 1-Wasserstein distance and Assumption \ref{assumption:lipschitz-rewards-and-transitions}.
\end{proof}

\subsection{Covering-related lemmas}


\begin{lemma}
	\label{lemma:lipschitz-covering}
	Let $\cF_L$ be the set of $L$-Lipschitz functions from the metric space $(\cX, \rho)$ to $[0,H]$. Then, its $\epsilon$-covering number with respect to the infinity norm is bounded as follows
	\begin{align*}
	\cN(\epsilon, \cF_L, \norm{\cdot}_\infty) \leq \pa{\frac{8H}{\epsilon}}^{\cN(\epsilon/(4L), \cX, \rho)}
	\end{align*}
\end{lemma}
\begin{proof}
	Let's build an $\epsilon$-covering of $\cF_L$. Let $\cC_{\cX} = \braces{x_1, \ldots, x_M}$ be an $\epsilon_1$-covering of $(\cX, \rho)$ such that $\rho(x_i, x_j) > \epsilon_1$ for all $i,j \in [M]$ (\ie $\cC_{\cX}$ is also an $\epsilon_1$-packing). Let $\cC_{[0, H]} = \braces{y_1, \ldots, y_N}$ be an $\epsilon_2$-covering of $[0, H]$. For any function $p: [M] \to [N]$, we build a $2L$-Lipschitz function $\widehat{f}_p: \cX \to \RR$ as follows
	\begin{align*}
		\widehat{f}_p(x) = \min_{i \in [M]} \sqrbrackets{ y_{p(i)} + 2L \rho(x, x_i) }.
	\end{align*}
	Let $\epsilon_1 = \epsilon/(4L)$ and $\epsilon_2 = \epsilon/8$. We now show that the set $\cC_{\cF_L} \eqdef \braces{\widehat{f}_p: p \text{ is a function from } [M] \text{ to } [N]}$ is an $\epsilon$-covering of $\cF_L$.
	Take an arbitrary function $f \in \cF_{L}$. Let $p: [M] \to [N]$ be such that $ \abs{ f(x_i)- y_{p(i)} } \leq \epsilon_2 $ for all $i \in [M]$. For any $x \in \cX$, let $j \in [M]$ be such that $\rho(x, x_j) \leq \epsilon_1$. We have
	\begin{align*}
		\abs{ f(x) - \widehat{f}_p(x) } & \leq \abs{ f(x_j) - \widehat{f}_p(x_j) } + \abs{ f(x) - f(x_j) }+ \abs{ \widehat{f}_p(x_j) - \widehat{f}_p(x) } \\
		& \leq \abs{ f(x_j) - \widehat{f}_p(x_j) } + 3L\rho(x, x_j) \\
		& \leq \abs{ f(x_j) - y_{p(j)}} + \abs{y_{p(j)} - \widehat{f}_p(x_j) } + 3L\epsilon_1 \\
		& \leq \abs{y_{p(j)} - \widehat{f}_p(x_j) } + 3L\epsilon_1 + \epsilon_2 \,.
	\end{align*}
	Now, let's prove that $\widehat{f}_p(x_j) = y_{p(j)}$, which is true if and only if $y_{p(j)} \leq  y_{p(i)} + 2L \rho(x, x_i)$ for all $i \in [M]$. By definition of $p$ and the fact that $f$ is $L$-Lipschitz, we have $ y_{p(j)} \leq y_{p(i)} + L\rho(x_j, x_i) + 2\epsilon_2 \leq y_{p(i)} + 2 L\rho(x_j, x_i)$  for all $i \in [M]$, since $L \rho(x_j, x_i) > L \epsilon_1 = 2\epsilon_2$. Consequently,
	\begin{align*}
		\forall x, \; \abs{ f(x) - \widehat{f}_p(x) } \leq 3L\epsilon_1 + \epsilon_2 < \epsilon
	\end{align*}
	which shows that $\cC_{\cF_L}$ is indeed an $\epsilon$-covering of $\cF_L$ whose carnality is bounded by $N^M$. To conclude, we take $\cC_{[0, H]} = \braces{0, \epsilon_2, \ldots, N\epsilon_2}$ for $N = \ceil{H/\epsilon_2}$ and $\cC_{\cX}$ such that $\abs{\cC_{\cX}} = M = \cN(\epsilon_1, \cX, \rho)$.

	For $H=1$, this result is also given by \cite{Gottlieb2017}, Lemma 5.2.
\end{proof}

\begin{lemma}
	\label{lemma:covering}
	Let $(\stateactionspace, \distfunc)$ be a metric space and $(\Omega, \cT, \bP)$ be a probability space. Let $F$ and $G$ be two functions from $\stateactionspace \times \Omega$ to $\RR$ such that $\omega \to F(x, a, \omega)$ and $\omega \to G(x, a, \omega)$ are random variables. Also, assume that $(x,a) \to F(x,a,\omega)$ and $(x,a) \to G(x,a,\omega)$ are  $L_F$ and $L_G$-Lipschitz, respectively, for all $\omega \in \Omega$. If
	\begin{align*}
	\forall (x, a), \quad  \prob{\omega \in \Omega: G(x, a, \omega) \geq F(x, a, \omega)} \leq \delta
	\end{align*}
	then
	\begin{align*}
	\prob{\omega \in \Omega: \exists (x,a), \; G(x, a, \omega) \geq F(x, a, \omega) + (L_G + L_f)\epsilon} \leq \delta \cN(\epsilon, \stateactionspace, \distfunc).
	\end{align*}
\end{lemma}
\begin{proof}
	Let $C_\epsilon$ be an $\epsilon$-covering of $(\stateactionspace, \distfunc)$ and let
	\begin{align*}
	(x_\epsilon, a_\epsilon) \eqdef \argmin_{(x', a') \in C_\epsilon} \dist{(x', a'), (x, a)}.
	\end{align*}
	Let $E \eqdef \braces{\omega \in \Omega: \exists (x,a), \; G(x, a, \omega) \geq F(x, a, \omega) + (L_G + L_f)\epsilon}$. In $E$, we have, for some $(x,a)$,
	\begin{align*}
	G(x^\epsilon, a^\epsilon, \omega) + L_G \epsilon \geq G(x, a, \omega) \geq F(x, a, \omega) + (L_G + L_f)\epsilon \geq F(x^\epsilon, a^\epsilon, \omega) + L_G\epsilon.
	\end{align*}

	Hence, in $E$, there exists $(x, a)$ such that:
	\begin{align*}
	G(x^\epsilon, a^\epsilon, \omega) \geq F(x^\epsilon, a^\epsilon, , \omega)
	\end{align*}
	and
	\begin{align*}
	\prob{E} & \leq \prob{\omega \in \Omega: \exists (x^\epsilon,a^\epsilon) \in C_\epsilon, \; G(x^\epsilon, a^\epsilon, \omega) \geq F(x^\epsilon, a^\epsilon, \omega)}\\
	& \leq \sum_{(x^\epsilon,a^\epsilon) \in C_\epsilon} \prob{\omega \in \Omega: G(x^\epsilon, a^\epsilon, \omega) \geq F(x^\epsilon, a^\epsilon, , \omega)} \leq \sum_{(x^\epsilon,a^\epsilon) \in C_\epsilon} \delta
	\end{align*}
	which gives us $\prob{E} \leq \delta \cN(\epsilon, \stateactionspace, \distfunc)$.
\end{proof}

\subsection{Technical lemmas}

We state and prove three technical lemmas that help controlling some of the sums that appear in our regret analysis.

\begin{lemma}
	\label{lemma:kernel-bias}
	Consider a sequence of non-negative real numbers $\braces{z_s}_{s=1}^t$ and let $\kernelfunc: \RR_+ \to [0, 1]$ satisfy Assumption \ref{assumption:kernel-behaves-as-gaussian}.  Let
	\begin{align*}
	w_s \eqdef \kernelfunc\pa{\frac{z_s}{\sigma}}  \; \mbox{ and } \; \widetilde{w}_s \eqdef \frac{w_s}{\beta + \sum_{s'=1}^t w_{s'}}.
	\end{align*}
	for $\beta > 0$. Then, for $t \geq 1$, we have
	\begin{align*}
	\sum_{s=1}^t \widetilde{w}_s z_s \leq 2\sigma \pa{1 + \sqrt{\log (C_1^g t/\beta + e)}}.
	\end{align*}
\end{lemma}
\begin{proof}
	We split the sum into two terms:
	\begin{align*}
	\sum_{s=1}^t \widetilde{w}_s z_s &= \sum_{s: z_s < c} \widetilde{w}_s z_s + \sum_{s: z_s \geq c} \widetilde{w}_s z_s  \leq c + \sum_{s: z_s \geq c} \widetilde{w}_s z_s
	\end{align*}
	From Assumption \ref{assumption:kernel-behaves-as-gaussian}, we have $w_s \leq C_1^g \exp\pa{-z_s^2/(2\sigma^2)}$. Hence, $ \widetilde{w}_s \leq (C_1^g/\beta) \exp\pa{-z_s^2/(2\sigma^2)}$, since $\beta + \sum_{s'=1}^t w_{s'} \geq \beta$.

	We want to find $c$ such that:
	\begin{align*}
	z_s \geq c \implies \frac{C_1^g}{\beta}\exp\pa{-\frac{z_s^2}{2\sigma^2}} \leq \frac{1}{t}\frac{2\sigma^2}{z_s^2}
	\end{align*}
	which implies, for $z_s \geq c$, that $\widetilde{w}_s \leq \frac{1}{t}\frac{2\sigma^2}{z_s^2}$.

	Let $ x = z_s^2/2\sigma^2 $. Reformulating, we want to find a value $c'$ such that $C_1^g\exp(-x) \leq \beta/(xt)$ for all $x \geq c'$. Let $c' = 2\log (C_1^gt/\beta + e)$. If $x \geq c'$, we have:
	\begin{align*}
	& \frac{x}{2} \geq \log \pa{\frac{C_1^g t}{\beta}+e}  \implies x \geq \frac{x}{2} + \log \pa{\frac{C_1^g t}{\beta}+e}
	\implies x \geq \log x + \log (C_1^gt/\beta+e)\\
	& \implies (C_1^g/\beta)\exp(-x) \leq 1/(xt)
	\end{align*}
	as we wanted. Hence, we choose $c' = 2\log (C_1^g t/\beta +e)$.

	Now, $x \geq c'$ is equivalent to $z_s \geq \sqrt{2\sigma^2 c'} = 2\sigma\sqrt{\log (C_1^gt/\beta+e)}$. Therefore, we take $c = 2\sigma\sqrt{\log (C_1^g t/\beta)}$, which gives us
	\begin{align*}
	\sum_{s: z_s \geq c} \widetilde{w}_s z_s & \leq \sum_{s: z_s \geq c} \frac{1}{t}\frac{2\sigma^2}{z_s^2}z_s \leq \frac{2\sigma^2}{t}\sum_{s: z_s \geq c} \frac{1}{z_s} \leq \frac{2\sigma^2}{c} \frac{\abs{\braces{s: z_s \geq c}}}{t} \leq \frac{2\sigma^2}{c}
	\end{align*}

	Finally, we obtain:
	\begin{align*}
	\sum_{s=1}^t \widetilde{w}_s z_s & \leq c + \sum_{s: z_s \geq c} \widetilde{w}_s z_s  \leq c + \frac{2\sigma^2}{c} \\ 
	& =  2\sigma\sqrt{\log (C_1^g t/\beta+e)} + \frac{\sigma}{\sqrt{\log (C_1^gt/\beta+e)}} \leq 2\sigma\pa{1 + \sqrt{\log (C_1^g t/\beta+e)}}
	\end{align*}
\end{proof}

\begin{lemma}
	\label{lemma:lipschitz-constant-mean-and-bonuses}
	Let $\braces{y_s}_{s=1}^t$ be a sequence of real numbers and let $\sigma > 0$.For $z \in \RR_+^t$, let
	\begin{align*}
	f_1(z) \eqdef \frac{\sum_{s=1}^t \kernelfunc(z_s/\sigma) y_s}{\beta + \sum_{s=1}^t \kernelfunc(z_s/\sigma)}, \quad
	f_2(z) \eqdef \sqrt{\frac{1}{\beta + \sum_{s=1}^t \kernelfunc(z_s/\sigma)}} \; \mbox{ and }
	f_3(z) \eqdef \frac{1}{\beta + \sum_{s=1}^t \kernelfunc(z_s/\sigma)}.
	\end{align*}
	Then, $f_1$, $f_2$ and $f_3$ are Lipschitz continuous with respect to the norm $\norm{\cdot}_\infty$:
	\begin{align*}
		\Lip{f_1} \leq \frac{2 C_2^g t (\max_s \abs{y_s})}{\beta \sigma}
		,\quad 
		\Lip{f_2} \leq \frac{C_2^g t}{2\sigma \beta^{3/2}}
		,\quad
		\Lip{f_3} \leq \frac{C_2^g t}{\sigma \beta^2}
	\end{align*}
	where $\Lip{f_i}$ denotes the Lipschitz constant of $f_i$, for $i\in\braces{1,2,3}$.
\end{lemma}
\begin{proof}
	Using Assumption \ref{assumption:kernel-behaves-as-gaussian}, the partial derivatives of $f_1$ and $f_2$ are bounded as follows
	\begin{align*}
	\abs{\frac{\partial f_1(z)}{\partial z_s}} & \leq \frac{1}{\sigma}\frac{\abs{\kernelfunc'(z_s/\sigma)} \abs{y_s}}{\beta + \sum_{s=1}^t \kernelfunc(z_s/\sigma)}
	+ \frac{1}{\sigma}\frac{\sum_{s=1}^t \kernelfunc(z_s/\sigma) \abs{y_s}}{\pa{\beta + \sum_{s=1}^t \kernelfunc(z_s/\sigma)}^2}\abs{\kernelfunc'(z_s/\sigma)} \leq \frac{2 C_2^g}{\beta \sigma}\max_s \abs{y_s}\\
	\abs{\frac{\partial f_2(z)}{\partial z_s}} & \leq \frac{1}{2\sigma}\frac{\abs{\kernelfunc'(z_s/\sigma)}}{\pa{\beta + \sum_{s=1}^t \kernelfunc(z_s/\sigma)}^{3/2}} \leq \frac{C_2^g}{2\sigma \beta^{3/2}} \\
	\abs{\frac{\partial f_3(z)}{\partial z_s}} & \leq \frac{1}{\sigma}\frac{\abs{\kernelfunc'(z_s/\sigma)}}{\pa{\beta + \sum_{s=1}^t \kernelfunc(z_s/\sigma)}^2} \leq \frac{C_2^g}{\sigma \beta^2}.
	\end{align*}

	Therefore,
	\begin{align*}
	\norm{\nabla f_1(z)}_1 \leq \frac{2 C_2^g t (\max_s \abs{y_s})}{\beta \sigma}, \quad 
	\norm{\nabla f_2(z)}_1 \leq \frac{C_2^g t}{2\sigma \beta^{3/2}}, \quad
	\norm{\nabla f_3(z)}_1 \leq \frac{C_2^g t}{\sigma \beta^2}
	\end{align*}
	and the result follows from the fact that $\abs{f_i(z_1) - f_i(z_2)} \leq \sup_z\norm{\nabla f_i(z)}_1\norm{z_1-z_2}_\infty$ for $i \in \braces{1, 2, 3}$.
\end{proof}

\begin{lemma}
	\label{lemma:aux-sum-of-bonuses}
	Consider a sequence $\braces{a_n}_{n\geq 1}$ of non-negative numbers such that $a_m \leq c$ for some constant $c > 0$. Let $A_t = \sum_{n=1}^{t-1} a_n$. Then, for any $b > 0$ and any $p > 0$,
	\begin{align*}
	\sum_{t=1}^T \frac{a_t}{(1 + b A_t)^p} \leq c + \int_0^{A_{T+1}-c} \frac{1}{ (1 + bz)^p }\mathrm{d}z
	\end{align*}
\end{lemma}

\begin{proof}
	Let $n \eqdef \max\braces{t: a_1 + \ldots + a_{t-1} \leq c}$. We have
	$\sum_{t=1}^{n-1} \frac{a_t}{(1 + b A_t)^p} \leq \sum_{t=1}^{n-1} a_t \leq c$
	and, consequently,
	\begin{align*}
	\sum_{t=1}^T \frac{a_t}{(1 + b A_t)^p} & \leq c + \sum_{t=n}^T \frac{a_t}{(1 + b A_t)^p} =  c + \sum_{t=n}^T \frac{A_{t+1} - A_t}{(1 + b A_t)^p} \\
	& = c + \sum_{t=n}^T \frac{A_{t+1} - A_t}{(1 + b A_{t+1} - b a_t)^p} \leq c + \sum_{t=n}^T \frac{A_{t+1} - A_t}{(1 + b (A_{t+1}-c))^p} \\
	& = c + \sum_{t=n}^T \int_{A_t}^{A_{t+1}}\frac{1}{(1 + b (A_{t+1}-c))^p} \mathrm{d}z
	 \leq c + \sum_{t=n}^T \int_{A_t}^{A_{t+1}}\frac{1}{(1 + b (z-c))^p} \mathrm{d}z \\
	& = c + \int_{A_n}^{A_{T+1}}\frac{1}{(1 + b (z-c))^p} \mathrm{d}z \leq  c + \int_{c}^{A_{T+1}}\frac{1}{(1 + b (z-c))^p}\mathrm{d}z \,.
	\end{align*}
\end{proof}

\section{Experimental setup}
\label{sec:app-experiments}

For \lipucbvi, we used the following simplified exploration bonuses:
\begin{align*}
\bonus_h^k(x, a) = \frac{1}{\sqrt{\gencount_h^k(x, a)}} + \frac{H - h +1}{\gencount_h^k(x, a)}.
\end{align*}
The same bonus was used for the baselines, except that $\gencount_h^k(x, a)$ was replaced by $\mathbf{N}_h^k(I(x), a) = \max\pa{1, \sum_{s=1}^{k-1}\indic{I(x_h^s) = I(x), a_h^s = a}}$
where $I(x)$ is the index of the discrete state corresponding to the continuous state $x$.

We used the Euclidean distance on the states and the Gaussian kernel function  $\kernelfunc(z) = \exp(-z^2/2)$. The regularization was taken as $\beta = 0.01$.

Additionally, in \lipucbvi, we used representative states \cite{kveton2012kernel, barreto2016practical} to merge states that are at a distance smaller than $0.05$ from each other, which provides a great improvement in the runtime of the algorithm.
\end{document}